\documentclass[10pt]{article}
\usepackage[margin=1in]{geometry}

\newif\ifopt
\optfalse         

\newcommand{\opt}[1]{\ifopt #1\fi}
\newcommand{\notopt}[1]{\ifopt\else #1\fi}

\usepackage{graphicx} 
\usepackage{tikz-cd}
\usepackage{amssymb}
\usepackage{dsfont}
\usepackage{tabularx, booktabs}
\usepackage{amsthm}
\usepackage{mathtools}
\usepackage{algorithm, algpseudocodex}
\notopt{\usepackage[colorlinks=true, allcolors=blue]{hyperref}}
\usepackage[capitalize]{cleveref}
\usepackage[round]{natbib}
\let\cite\citep

\usepackage{enumitem}

\usepackage{multirow}
\usepackage{amsmath,amsfonts,graphicx}
\usepackage{xfrac}

\usepackage{commands}
\usepackage[dvipsnames]{xcolor}

\newtheorem{theorem}{Theorem}  
\newtheorem{lemma}[theorem]{Lemma}

\newtheorem*{theorem*}{Theorem}
\newtheorem*{lemma*}{Lemma}

\theoremstyle{definition}
\newtheorem{definition}[theorem]{Definition}

\theoremstyle{remark}

\newenvironment{proofsketch}[1][\proofname]{%
  \begin{proof}[#1 (sketch)]%
}{%
  \end{proof}%
}

\renewcommand{\epsilon}{\varepsilon}

\newcommand\E{\mathbb{E}}

\newcommand{\W}{\mathcal{W}}

\newcommand{\rr}{\mathbb{R}}
\newcommand{\R}{\mathbb{R}}

\newcommand\inprod[2]{\langle #1,\,#2 \rangle}

\newcommand{\curly}[1]{ {\left\{ #1 \right\}}}
\newcommand{\roundy}[1]{ {\left( #1 \right)}}
\newcommand{\squary}[1]{ {\left[ #1 \right]}}
\newcommand{\abs}[1]{ {\left | #1 \right |}}

\newcommand{\barconst}{\nu}
\newcommand{\reg}{{\rm Regret}}

\title{The Hidden Cost of Approximation in Online Mirror Descent}
\author{Ofir Schlisselberg  \and Uri Sherman \and Tomer Koren \and Yishay Mansour}

\newcommand{\mail}[1]{\href{mailto:#1}{\color{blue} #1}}
\newcommand{\blackfootnote}[1]{%
  \begingroup
  \hypersetup{allcolors=black}%
  \footnote{#1}%
  \endgroup
}

\makeatletter
\newcommand{\ForceCrefTypeInEnv}[2]{%
  \AddToHook{env/#1/begin}{%
    \let\cref@oldlabel\label
    \def\label##1{\cref@oldlabel[#2]{##1}}%
  }%
  \AddToHook{env/#1/end}{\let\label\cref@oldlabel}%
}
\makeatother

\ForceCrefTypeInEnv{lemma}{lemma}
\ForceCrefTypeInEnv{definition}{definition}
\ForceCrefTypeInEnv{proposition}{proposition}
\ForceCrefTypeInEnv{claim}{claim}
\ForceCrefTypeInEnv{corollary}{corollary}
\ForceCrefTypeInEnv{example}{example}
\crefname{lemma}{lemma}{lemmas}
\Crefname{lemma}{Lemma}{Lemmas}
\crefname{definition}{definition}{definitions}
\Crefname{definition}{Definition}{Definitions}
\crefname{proposition}{proposition}{propositions}
\Crefname{proposition}{Proposition}{Propositions}
\crefname{claim}{claim}{claims}
\Crefname{claim}{Claim}{Claims}
\crefname{corollary}{corollary}{corollaries}
\Crefname{corollary}{Corollary}{Corollaries}
\crefname{example}{example}{examples}
\Crefname{example}{Example}{Examples}

\author{Ofir Schlisselberg\blackfootnote{Tel Aviv University; \mail{ofirs4@mail.tau.ac.il}} \and
Uri Sherman\blackfootnote{Tel Aviv University; \mail{urisherman@mail.tau.ac.il}} \and
Tomer Koren\blackfootnote{Tel Aviv University and Google Research; \mail{ tkoren@tauex.tau.ac.il}} \and
Yishay Mansour\blackfootnote{Tel Aviv University and Google Research; \mail{mansour.yishay@gmail.com}}}

\begin{document}

\maketitle

\begin{abstract}
    Online mirror descent (OMD) is a fundamental algorithmic paradigm that underlies many algorithms in optimization, machine learning and sequential decision-making. 
    The OMD iterates are defined as solutions to optimization subproblems which, oftentimes, can be solved only approximately, leading to an \emph{inexact} version of the algorithm.
    Nonetheless, existing OMD analyses typically assume an idealized error free setting, thereby limiting our understanding of performance guarantees that should be expected in practice.
    In this work we initiate a systematic study into inexact OMD, and uncover an intricate relation between regularizer smoothness and robustness to approximation errors.
    When the regularizer is uniformly smooth,
    we establish a tight bound on the excess regret due to errors.
    Then, for barrier regularizers over the simplex and its subsets, we identify a sharp separation: negative entropy requires exponentially small errors to avoid linear regret, whereas log-barrier and Tsallis regularizers remain robust even when the errors are only polynomial. Finally, we show that when the losses are stochastic and the domain is the simplex, negative entropy regains robustness - but this property does not extend to all subsets, where exponentially small errors are again necessary to avoid suboptimal regret.
\end{abstract}

\section{Introduction}
Mirror Descent~\citep{nemirovsky1983problem,beck2003mirror} is a fundamental optimization paradigm that offers the flexibility to exploit the (typically non-Euclidean) intrinsic geometry of the optimization problem. 
The online variant (OMD; \citealp{shalev2012online,hazan2016introduction}) is a generalization of the basic framework adapted to the more general online learning setup \cite{zinkevich2003online}, where the goal of the learner is to minimize her \emph{regret}, defined as the cumulative loss minus the loss of the best fixed decision in hindsight.
Given a convex decision set $\cK\subset \R^d$, 
an initialization $w_1\in \cK$ and learning rate $\eta > 0$, the OMD steps $t=1, \ldots, T$ follow the update rule:
\begin{align}\label{eq:OMD}
    w_{t+1} = \argmin_{w\in \cK} \bigl\{  \eta \langle \ell_t, w \rangle + D_R(w \,\|\, w_t) \bigr\} ,
\end{align} 
where $\l_t$ is the loss at time $t$ and $D_R$ is the Bregman divergence associated with a regularizer \(R\colon \cK \to \R\) chosen by the learner.
Notable instances of OMD include online gradient descent \citep{zinkevich2003online} and the well known multiplicative weights method \citep{littlestone1994weighted, freund1997decision, arora2012multiplicative}, both of which are examples where the OMD update rule, namely the exact solution to the OMD subproblem Eq.~(\ref{eq:OMD}), is given by a closed form expression (when operating over suitable decision sets). 

However, in many cases of interest, the OMD update rule does not admit a closed form solution, and therefore demands employing an auxiliary iterative optimization procedure that only produces \emph{approximate} minimizers of the respective OMD subproblems.
Notable examples include reinforcement learning algorithms that optimize over occupancy measures, which form a polyhedral subset of the simplex \citep{NIPS2013_68053af2,pmlr-v97-rosenberg19a,pmlr-v119-jin20c}; 
generic online convex optimization algorithms that rely on OMD updates \citep{5044958,10.5555/2968826.2968914,pmlr-v108-ito20a}; 
and algorithms defined over the simplex that use barrier regularization other than negative entropy, such as in adversarial bandits \citep{NIPS2015_5caf41d6,zimmert2021tsallis} and portfolio selection \citep{luo2018efficient}.  
Somewhat surprisingly, however, the existing literature lacks a systematic study of the effect these approximations have on the final regret guarantee, with prior art focusing on particular problem instances at best \cite{schmidt2011convergence, villa2013accelerated, dixit2019online,choi2023inexact}.
\opt{(Due to space constraints, discussion of additional related work is deferred to \Cref{sec:related work}.)}

In this work, we initiate a systematic study into the robustness of OMD to approximations, aimed at understanding the interplay between regularization, quality of approximations, and regret.
Our results uncover a direct link between robustness of inexact OMD and smoothness properties of the regularizer being used.
For uniformly smooth regularizers, we establish that robustness to approximation errors is directly governed by the smoothness parameter. For the more prevalent non-smooth regularizer case, we demonstrate that OMD with negative entropy regularization is prone to incurring \emph{linear} regret unless the approximation errors are made \emph{exponentially small} in the number of steps; and in contrast, that for other barrier regularizers such as the log-barrier and Tsallis entropy, polynomially small errors suffice to obtain optimal regret.
We then further investigate more carefully when non-robustness with the negative entropy arises. We show that when the losses are stochastic (i.i.d.), negative entropy over the simplex becomes robust and polynomially small errors are sufficient. On the other hand, we demonstrate this robustness may break even with i.i.d.\ losses when optimizing over a \emph{subset} of the simplex, where again, exponentially small errors are necessary to avoid suboptimal regret.  

\subsection{Summary of contributions}
In more detail, our contributions are summarized as follows.
\begin{itemize}
    \item 
    First, when the regularizer $R$ is uniformly smooth over the domain $\cal K$ with smoothness parameter $\beta$, we establish a tight $\Theta(TD\sqrt{\beta\epsilon}/\eta)$ bound on the excess regret due to $\epsilon$-approximation errors,
    where $D$ is the diameter of $\cal K$ with respect to the relevant norm.
    E.g., for the typical setting $\eta = \Theta(1/\sqrt{T})$ this implies that errors should be as small as $\epsilon = O(1/T^2)$ so as to recover optimal $O(\sqrt{T})$ regret.

    \item
    We then move on to consider common non-smooth regularizers, such as the negative entropy, Tsallis entropies, and the log-barrier, focusing on the simplex and its subsets as decision sets. We observe a sharp dichotomy between the negative entropy and other regularizers in terms of robustness to approximations: on the one hand, for the negative entropy we show that an \emph{exponentially small} error  $\epsilon = \Omega(\eta e^{-\eta T})$ could already lead to \emph{linear regret}, even when the domain is the simplex; and on the other hand, for Tsallis Entropies and the log-barrier over the simplex or a subset thereof, we prove that a \emph{polynomially small} error, e.g $\epsilon = O(\eta^2/(T^2d^2))$ for log-barrier, suffices for maintaining the same order of regret.
    
    \item
    Finally, we revisit the robustness to approximations with the negative entropy in the stochastic~(i.i.d.) setting.
    Over the simplex and with $\eta = \widetilde O(1/\sqrt{T})$, we show that a polynomially small error $\epsilon = O(1/ (d^2T^4))$ suffices for obtaining optimal regret with high probability, as opposed to the exponentially small error required in the non-stochastic case.
    However, this robustness does not extend more generally to proper subsets of the simplex: we construct a setting where OMD with negative entropy exhibits an excess term of 
    $
        \Omega( T\sqrt{\eta / \log(1/\epsilon)} )
    $
    leading to $\widetilde \Omega(T^{2/3})$  regret for any step size 
    unless $\epsilon$ is exponentially small in $T$.
\end{itemize}
At a conceptual level, our analysis reveals that compounding errors play a central role in OMD's robustness to inexact updates. Since the per time step subproblem directly depends on the previous iterate, approximation errors propagate between rounds and lead to subtle optimization dynamics. This should be contrasted with the closely related Follow-The-Regularized-Leader (FTRL) algorithm \citep[e.g.,][]{shalev2012online,hazan2016introduction}, which re-optimizes against the cumulative loss at each round, and thus, each optimization round is independent of inexactness introduced in previous rounds. And indeed, for FTRL it is straightforward to prove that approximation errors have only a minor effect; for more details, see \Cref{sec:ftrl}\footnote{We note that even though the iterates of FTRL and OMD coincide in some cases in the exact case, it isn't true for the approximate case.}.

In addition, our results for the smooth case (\Cref{thm:lb_smooth,thm:ub_smooth}) provide a tight characterization that is immediately applicable to a common technique where OMD is instantiated over a \emph{shrunk} simplex (or subset thereof), where coordinates are bounded away from zero (as suggested by \cite{dick2014online}). In this case, a uniform bound for the smoothness parameter immediately follows as the regularizer domain becomes compact. Interestingly, our results for the non-smooth case reveal that while this technique may be necessary to cope with fragility of negative entropy (\Cref{thm:lb_barrier_adversarial}), it is not necessary for other barrier regularizers as they induce optimization dynamics where the iterates \emph{naturally} stay bounded away from zero (see \Cref{thm:ub_nonentropy} and the discussion that follows).

Finally, we note that while our study focuses on the linear setup, all our results for the adversarial setting immediately carry to the general convex case via a standard reduction (e.g., \citealt{cesa2006prediction}).

 \begin{table}[t]
 \centering
\begin{tabular}{c|c|c|c|c}
\toprule
\textbf{Regime} & \textbf{Decision set} & \textbf{Regularizer} & \textbf{Tolerated $\boldsymbol{\epsilon}$} & \textbf{Polynomial?} \\
\midrule
Adversarial & convex & $\beta$-smooth & {$\eta^4/\beta$} & \checkmark\\
Adversarial & simplex subset & $\barconst$-barrier ($\barconst > 1$) & {$\eta^4\roundy{\eta Td}^{-\nu/\nu-1}$} & \checkmark\\
Adversarial & simplex & negative entropy & {$e^{-\eta T}$} & \ding{55}\\
Stochastic & simplex & negative entropy & {$d^{-2}T^{-4}$} & \checkmark\\
Stochastic & simplex subset & negative entropy & {$e^{-1/\eta}$} & \ding{55}\\
\bottomrule
\end{tabular}
\caption{Summary of contribution. The required $\epsilon$ column neglects low order terms.}
\end{table}

\subsection{Related work}\label{sec:related work}

Mirror descent \cite{nemirovskij1983problem, beck2003mirror} and the online convex optimization framework \citep{zinkevich2003online} have been central to the study of machine learning and optimization in the last decades.
There exist many excellent books and surveys that provide thorough introductions to (online) mirror descent in its fundamental (i.e., exact, error free) form \citep{shalev2012online, bubeck2015convex, hazan2016introduction, beck2017first}. Somewhat surprisingly, there hardly exist any works that study inexact mirror descent in the general stochastic or online setup.

In the classical (offline) optimization setup where the objective function is smooth, mirror descent coincides with a special case of the Bregman proximal gradient method (BPGM; \cite{bauschke2017descent, lu2018relatively}, see also \cite{teboulle2018simplified}).
The BPGM is a generalization of the proximal gradient method \citep{rockafellar1976monotone} where a Bregman divergence replaces the norm proximity regularizer, and the objective is required to satisfy the weaker \emph{relative smoothness} property \citep{bauschke2017descent}.
The BPGM and mirror descent coincide when the non-smooth part in the composite objective is the indicator function for the decision set.
In contrast to online or stochastic mirror descent in the general case, inexact versions of the BPGM (and thus offline mirror descent in the smooth case have been subject to several recent works. 
The majority of these study the Euclidean case (i.e., the proximal gradient method) with or without acceleration, e.g., \cite{schmidt2011convergence, villa2013accelerated, zhou2022acceleration, ahookhosh2024high}.
Some works study the online case 
\cite{dixit2019online} with the euclidean regularizer, and some further generalize to the online BPGM but with smooth regularizers \citep{choi2023inexact}.

There is also a recent line of works that study the (offline) BPGM in its general form (i.e., without making assumptions on the regularizer). These mostly focus on designing variants of the basic method that incorporate some mechanism to cope with the proximal subproblem approximation errors \citep{rebegoldi2018bregman, kabbadj2020inexact, stonyakin2021inexact, yang2025inexact}---which is to be contrasted with characterizing convergence in terms of the ad-hoc approximation errors.
As one example, the work of \citet{kabbadj2020inexact} establishes that the inexact BPGM achieves the same rate of the exact version (aka NoLips; \citealp{bauschke2017descent}) as long as the approximation errors are smaller than the Bregman distance to the previous iterate.
More recently, \citet{yang2025inexact} propose variants with several advantages at the expense of a somewhat more involved subproblem optimization procedure.

Finally, the work of \citet{guigues2021inexact} is one of the only examples (to our best knowledge) of papers that study an inexact version of stochastic mirror descent, albeit one that relates to a particular (non-general) instantiation of the algorithm.

\section{Preliminaries}
We consider the standard online linear optimization setup, where at each round \(t = 1, 2, \ldots, T\), the learner selects a point \(w_t\) from a convex decision set \(\mathcal{K} \subset \mathbb{R}^d\), and then observes a loss vector \(\ell_t\). The performance of the learner is measured in terms of her \emph{regret} with respect to a fixed comparator point \(w \in \mathcal{K}\), defined as follows:
\[
\reg(w) = \sum_{t=1}^T \langle \ell_t, w_t \rangle - \sum_{t=1}^T \langle \ell_t, w \rangle.
\]
We denote by \(w^* \in \arg\min_{w \in \mathcal{K}} \sum_{t=1}^T \langle \ell_t, w \rangle\) the best fixed decision in hindsight.

\paragraph{Inexact Online Mirror Descent.}
We let \(R : \mathcal{K} \to \mathbb{R}\) denote a differentiable regularizer which we assume to be $1$-strongly convex w.r.t.~a norm \(\|\cdot\|\).
The Bregman divergence associated with \(R\) is  defined as:
\[
D_R(w \,\|\, w') = R(w) - R(w') - \langle \nabla R(w'), w - w' \rangle.
\]
We say that a sequence $\curly{w_t}_{t=1}^T$ is an \emph{$\epsilon$-approximate OMD trajectory} if, for every $t$, $w_{t+1}$ approximately minimizes the round $t$ OMD objective (see Eq.~\ref{eq:OMD}) 
 $\phi_t (w) \eqq \eta \langle \ell_t, w \rangle + D_R(w \,\|\, w_t)$, up to $\epsilon$ additive error:
\[
\phi_t(w_{t+1}) \le \min_{w \in \mathcal{K}} \phi_t(w) + \epsilon.
\]
Regret bounds for OMD typically depend on the \emph{diameter} of $\cK$ with respect to the norm $\norm{\cdot}$, given by $D=\max_{w,w'\in\cK} \norm{w-w'}$ and demands $\norm{\ell_t}_* \le 1$ for all $t\in[T]$. 

\paragraph{Barrier Regularization.}
A particular focus of this work is on prototypical barrier regularizers, used extensively in cases where $\cK$ is the probability simplex  $\Delta_d \eqq \cbr[s]{p\in \R^d : p^i \geq 0, \sum_{i=1}^d p^i = 1} $ (or a subset thereof).
\begin{definition}[coordinate separable barrier regularizers]\label{def:nu_barrier}
We say $R\colon \cK \to \R$ is a coordinate separable barrier\footnote{Strictly speaking, these are barriers for the positive orthant in $\R^d$.} regularizer with parameter $\barconst\geq1$ (or simply a $\barconst$-barrier) 
if there exists a twice-differentiable function \(r \colon [0,1] \to \mathbb{R}\) and $c_1,c_2>0$ such that:
\[
R(w) = \sum_{i=1}^d r(w^i),
\quad \text{and} \quad 
\frac{c_1}{x^\barconst} \leq r''(x) \leq \frac{c_2}{x^\barconst} \quad \text{for all } x \in (0,1].
\]
\end{definition}
These conditions ensure that the regularizer imposes a barrier-like growth as components of \(w\) approach zero, which plays a crucial role in controlling the optimization dynamics near the boundary of the positive orthant.
This class captures several widely used regularizers, including: 
\begin{itemize}[nosep]
    \item \emph{Negative Entropy:} \(r(x) = x \log x\), for which $\barconst=1$;
    \item \emph{Tsallis Entropy:} \(r(x) = \frac{x - x^q}{1 - q}\) for \(q \in (0,1)\), where $1<\barconst<2$;
    \item \emph{Log-Barrier:} \(r(x) = -\log x\), which corresponds to $\barconst=2$.
\end{itemize}
The parameter $\barconst$ will turn out to be directly associated with the robustness of OMD with $\barconst$-barrier regularization to approximation errors.

\notopt{
\paragraph{Additional notation.}  
We denote by \(\ell_{t_1:t_2} = \sum_{t = t_1}^{t_2} \ell_t\) the cumulative loss vector over the interval \([t_1, t_2]\). For any vector \(v \in \mathbb{R}^d\), we use \(v^i\) to denote its \(i\)-th coordinate. For example, \(\ell_t^i\) refers to the \(i\)-th component of the loss vector at time \(t\), and \(w_t^i\) denotes the \(i\)-th component of the learner's decision at time \(t\).
}

\section{Smooth regularizers}\label{sec:smooth}
We begin by establishing tight upper and lower bounds for approximate OMD with smooth regularizers,\footnote{A function $R$ is said to be $\beta$-smooth with respect to a norm $\|\cdot\|$ if its gradient is $\beta$-Lipschitz; 
$
    \|\nabla R(x) - \nabla R(y)\|_* \le \beta \|x-y\|
$ for all $x,y \in \cK$, where $\norm{\cdot}_*$ is the norm dual to $\norm{\cdot}$.}
over an arbitrary convex domain $\cK \subseteq \R^d$.
Our first theorem provides an upper bound that builds on the following key property of smooth functions: approximate minimization implies that first-order optimality conditions hold up to an error proportional to the square root of the sub-optimality times the smoothness parameter. The formal proofs for this section is given in \Cref{apx:smooth}.

\begin{theorem}
\label{thm:ub_smooth}
Let $\cK \subseteq \R^d$ be a convex set with diameter $D$, and let $R \colon\cK\to \rr$ be a $\beta$-smooth regularizer over $\cK$.
Then, for any loss sequence $\ell_1, \dots, \ell_T$ such that $\norm{\ell_t}_* \le 1$ for all $t\in[T]$, the regret of any $\epsilon$-approximate OMD trajectory with $\epsilon \le D^2/2$ compared to any $w \in \cK$ is bounded as:
\begin{align*}
    \reg(w) = O\roundy{\frac{1}{\eta}D_R(w, w_1) + T\eta + \frac{TD\sqrt{\beta\epsilon}}{\eta}}
    .
\end{align*}
\end{theorem}
The proof follows the standard OMD analysis, replacing exact optimality with \emph{approximate} optimality conditions.
Indeed, for any $\beta$-smooth convex objective $f\colon \cK \to \R$, if $f(\hat w) - \argmin_{w\in \cK}f(w) \le  \epsilon$, then one can show that (see \Cref{lem:epsilon optimality conditions}):
\begin{align}\label{eq:approx_oc_ub_smooth}
    \abs{\inprod{\nabla f(\hat{w})}{w - \hat{w}}} \le D\sqrt{
    2\beta\epsilon}.
\end{align}
Applying the above on $\phi_t$ for every $t$, and carrying the errors in the standard OMD analysis, gives the claimed result.

We note that \cref{thm:ub_smooth} provides sharper dependence on $\beta$ compared to a similar result of \cite{choi2023inexact}. 
This bound is in fact tight, even in the simple case of OMD with Euclidean regularization and constant losses, as shown next.

\begin{theorem}
\label{thm:lb_smooth}
Let $\beta,\epsilon, D>0$, and consider $\epsilon$-approximate OMD
over $\cK=[0, D]$ with the $\beta$-smooth regularizer $R(\cdot)=\frac\beta 2\|\cdot\|^2_2$.
Then there exists a loss sequence, 
an $\epsilon$-approximate OMD trajectory
and $w\in\cK$ such that: 
\begin{align*}
    \reg(w) 
    =
    \Omega\roundy{\min\biggl\{ \frac{TD\sqrt{\beta\epsilon}}{\eta}, DT \biggr\} }
    .
\end{align*}
\end{theorem}

To see why this is true,
consider the constant loss sequence $\ell_t=\min\cbr[b]{\sqrt{2\beta\epsilon}/\eta, 1}$ for all $t\in[T]$, and initialize the trajectory at $w_1 = D/2$. Then for every $t$, the loss is small enough so that $w_t$ itself is an $\epsilon$-minimizer of $\phi_t$; 
let $w_{t+1}^*$ be the exact minimizer of $\phi_t$, 
then by direct computation: 
\begin{align*}
    \phi_t(w_{t+1}^*)
    = \eta \abr{\l_t, w_t}
    - \epsilon
    = \phi_t(w_t) - \epsilon.
\end{align*}
As a result, the approximation error might prevents any update from changing the iterate, so the trajectory remains fixed at $w_t = w_1$ for all $t$. Consequently, the algorithm incurs the claimed regret.
We note that the underlying reason the above argument works is that for the Euclidean regularizer, in the setting of \Cref{thm:lb_smooth}, round $t$ approximate optimality conditions (Eq.~\ref{eq:approx_oc_ub_smooth}) are in fact tight. 

\section{Barrier regularizers}\label{sec:barrier}
\subsection{Adversarial losses}\label{sec:barrier_adv}
We next consider barrier regularizers the smoothness of which is not bounded uniformly over the domain $\cK$. 
Indeed, the spectrum of the Hessian of any $\barconst$-barrier (\Cref{def:nu_barrier}) is unbounded since $r''(x) \to \infty$ as $x \to 0$.
Interestingly, the robustness behavior of these barriers varies dramatically with $\barconst$: for negative entropy ($\barconst=1$), exponentially small errors are required, whereas for log-barrier or Tsallis regularizers ($\barconst>1$), polynomially small errors suffice. We begin with our lower bound for negative entropy given below.

\begin{theorem}
\label{thm:lb_barrier_adversarial}
Let $\cK=\Delta_d$, $d\ge2$, and $R$ be the negative entropy over $\cK$.
Suppose that the approximation error satisfies $\epsilon \ge 4\eta e^{-\eta T/3}$.
Then there exists a sequence of losses $\ell_1, \ldots, \ell_T\in[0,1]^d$ for which there exists an $\epsilon$-approximate OMD trajectory that suffers regret $\reg(w^*) =\Omega(T)$. 
\end{theorem}

The key idea in the analysis of \Cref{thm:lb_barrier_adversarial} is
to exploit the fact that 
the \emph{effective} smoothness of the regularizer---informally, the exact smoothness parameter on a given region---diverges at a rate inversely proportional to the iterate coordinates as they approach zero.
Indeed, our construction is such that the coordinates of the iterate become as small as $e^{-\eta T}$ (this follows from the closed form update equations), and thus reach the region of the domain where the effective smoothness is exponentially large.
Then, 
when the errors are not exponentially small, the same mechanism as in the smooth-regularizer lower bound applies: the iterate becomes stuck even under constant losses, leading to linear regret. A similar argument also shows that $\epsilon$ must be polynomially small in $d$; otherwise, the iterates can remain stuck at the initialization point (see \Cref{thm:dimension lower bound}). The exponential dependence of $\epsilon$ on the time horizon $T$ is in fact tight: if $\epsilon$ is exponentially small in $\eta T$, the standard regret guarantees are recovered.

\begin{theorem}
\label{thm:up_barrier_adversarial}
Let $\cK=\Delta_d$ and $R$ be the negative entropy over $\cK$.
Assume $\eta \le 1/16$ and $T\ge 3$, if 
$\epsilon \leq \tfrac{1}{6d} e^{-\eta T/2} \min\curly{\eta^4,T^{-2}},$
then for any loss sequence $\ell_1, \dots, \ell_T \in [-1,1]^d$, the regret of any \(\epsilon\)-approximate OMD 
trajectory compared to any $w \in \cK$ is bounded as:
\[
\mathrm{Regret}(w) \le \frac{1}{\eta} D_R(w, w_1) + O(\eta T),
\]
where big-$O$ hides only constant numerical factors. 
\end{theorem}
The proof is deferred to \Cref{sec:sketches}. 

We now turn our attention to $\barconst$-barrier regularizers with $\barconst > 1$. In this case, as it turns out, polynomially small errors suffice to \emph{naturally} keep the iterates bounded away from zero (by a polynomial margin).
\begin{theorem}
\label{thm:ub_nonentropy}
Let $\cK \subseteq \Delta_d$ be a polytope that contains the uniform distribution and the OMD is initialized there,%
\footnote{This assumption serves mainly to fix a natural starting point for OMD; a similar bound should hold for any reasonable initialization.}
let $R\colon \cK \to \R$ be a $\barconst$-barrier regularizer (cf.\ Definition \ref{def:nu_barrier}) with $\barconst>1$ and $\eta \le 1/(16c_1)$. 
If 
$$
    \epsilon 
    \leq \eta^4\min\curly{\tfrac{1}{c_2},c_2}\roundy{\tfrac{16\eta Td}{c_1} + 2(2d)^{\barconst-1}}^{-\frac{\barconst}{\barconst-1}}
    ,
$$ 
then for any loss sequence $\ell_1, \dots, \ell_T\in[-1,1]^d$, the regret of any \(\epsilon\)-approximate OMD trajectory compared to any $w \in \cK$ is bounded as:
\[
\mathrm{Regret}(w) \le \frac{1}{\eta} D_R(w, w_1) + O(\eta T).
\]
\end{theorem}
The principle underlying the analysis of \Cref{thm:ub_nonentropy} is as follows. Consider for purposes of illustration the one-dimensional interval $[0,1]$ with $w_1=1$. In this setting the OMD updates require no projection and the iterate dynamics can be inspected more simply:
\begin{align*}
    r'(w_{t}) &= r'(w_{t-1}) - \eta \ell_{t-1}
     = r'(w_1) - \eta \sum_{s=1}^{t-1} \ell_s,
     \intertext{which implies}
    -\tfrac{1}{w_t^{\barconst-1}} &\ge -1 - \eta T
    \quad\implies\quad
    w_t \ge (\eta T)^{-\frac{1}{\barconst-1}}.
\end{align*}
Namely, the iterates can only shrink polynomially in $T$, and as a result the effective smoothness grows polynomially. 
This allows the use of approximate first-order optimality conditions in the standard OMD analysis, and the regret may be bounded using the standard OMD proof. Note that this comes in contrast to the negative entropy case ($\barconst=1$, \Cref{thm:lb_barrier_adversarial}) where a similar argument in this simplified setting gives 
$
    \log(w_t) \ge 0-\eta T
    \implies
    w_t \ge e^{-\eta T}
$. 
Finally, the simplified setting considered above we did not account for the possibility that the errors themselves can pull the iterates closer to the boundary. Evidently, the approximation errors may potentially drive the iterates toward zero even when the exact dynamics would not, which further complicates the analysis.

\subsection{Improved robustness with stochastic losses}\label{sec:stochastic}

In the adversarial setting, we have seen that negative entropy requires exponentially small error to avoid linear regret, even over the simplex. Surprisingly,
this fragility does not persist for stochastic losses over the full simplex. For i.i.d.~stochastic losses, polynomially small approximation errors suffice to guarantee standard regret bounds with high probability.

\begin{theorem}
\label{thm:ub_barrier_stochastic}
Let $\cK=\Delta_d$ and $R$ be the negative entropy over $\cK$. For any $\delta>0$, suppose that $w_1=(1/d,1/d,\dots,1/d)$, $T\ge 256$, $\eta = \sqrt{\frac{\log (d)}{T}}$ and $\epsilon \le \frac{\delta}{ 6 d^2 T^4}$.
Then with probability $\geq 1-\delta$ over the choice of an i.i.d.~loss sequence $\ell_1, \dots, \ell_T\in[-1,1]^d$ 
the regret of any \(\epsilon\)-approximate OMD trajectory compared to any $w \in \cK$ is
$O(\sqrt{T\log(d)})$.
\end{theorem}

However, this robustness does not extend to general domains. Even with the same regularizer and similarly stochastic losses, restricting the domain to a polyhedral subset of the simplex can cause suboptimal regret unless the approximation error is exponentially small.

\begin{theorem}
\label{thm:lb_nonsimplex}
Consider approximate OMD with the negative entropy regularizer and stochastic losses.
Then, there exists a polytope $\cK \subseteq \Delta_d$ and a distribution of losses such that for any $\epsilon>0$, there exists an $\epsilon$-approximate trajectory and $w\in\cK$ such that:
 \begin{align*}
     \E\squary{\reg(w)} = \Omega\roundy{\frac{D_R(w,w_1)}{\eta} +T\sqrt{\frac{\eta}{\log\roundy{1/\epsilon}}}}
     .
 \end{align*}
\end{theorem}

One can see that any approximation error that is merely polynomial in $T$ leads to a sub-optimal regret lower bound of \(\widetilde{\Omega}(T^{2/3})\), even under an optimally tuned learning rate.

\section{Analysis overview}\label{sec:sketches}
In this section we sketch the proofs of the results from \Cref{sec:barrier}. We begin in \Cref{sec:balance} by introducing the balance framework, which serves as a unifying tool throughout the analysis. The balance of an OMD trajectory quantifies how “well-behaved” it is, measuring how much noise or fluctuation it exhibits. Our framework provides analytic tools for computing or bounding the balance of a trajectory, relating it to a certain notion of balance of the loss sequence, and further translating balance to other properties of the trajectory.

Next, we apply the balance framework in two ways. In \Cref{sec:non entropy} we show that for barrier regularizers other than negative entropy ($\barconst > 1$), the iterates cannot approach zero too closely, which directly yields a polynomial-error upper bound (\Cref{thm:ub_nonentropy}). Later, in \Cref{sec:entropy}, we analyze the negative entropy regularizer and establish 
a relationship between regret and balance of the loss sequence (\Cref{lem:balance lower bound,lem:balance upper bound}), leading to our main results for the adversarial (\Cref{thm:lb_barrier_adversarial,thm:up_barrier_adversarial}) and stochastic (\Cref{thm:ub_barrier_stochastic}) settings. 
The proof of \Cref{thm:lb_nonsimplex} is deferred to \Cref{sec:polytope lower bound}.

\subsection{Balance}\label{sec:balance}
To analyze the trajectory of OMD over general polytopes we introduce the notion of \emph{Balance}. 
We assume the polytope is given in standard form:
\[
    \cK = \{ w \in \mathbb{R}^d : Aw = b, \; w_i \ge 0 \ \forall i \in [d] \},
\]
where $A \in \mathbb{R}^{m \times d}$ with $m<d$ and $b \in \mathbb{R}^m$ 
(see, e.g., Eq.\ 4.28 in \citealp{Boyd_Vandenberghe_2004}). 
\begin{definition}
For every $v \in \ker(A)$ and $1 \le t_1 < t_2 \le T$ we define the balance of an OMD trajectory $w_1, \ldots, w_T \in \cK$ with respect to $t_1,t_2,v$ as follows:
\[
    B^v(t_1,t_2) = \langle \nabla R(w_{t_1}) - \nabla R(w_{t_2}), v \rangle.
\]
If for every $v \in \ker(A)$ and every $t_1,t_2$ we have $B^v(t_1,t_2) \le k \|v\|$, 
we say the trajectory is $k$-balanced w.r.t.~the norm $\norm{\cdot}$. 
\end{definition}

Our first lemma relates
 the variation of the loss sequence to the balance of the OMD iterates, which in turn will be used to establish properties of the trajectory leading to regret upper or lower bounds.
\begin{lemma}\label{lem:polytope balance}
Assume the OMD trajectory is exact, then for every $v\in \ker(A)$ and times $t_1,t_2$:
\begin{align*}
    B^v(t_1,t_2) = \eta \langle \ell_{t_1:t_2}, v \rangle
    .
\end{align*}
\end{lemma}

This motivates the definition that a sequence of losses $\{ \ell_t \}_{t=1}^T$ is $\alpha$-balanced w.r.t norm $\norm{\cdot}$ if for every $v \in \ker(A)$ we have 
$\langle \ell_{t_1:t_2}, v \rangle \le \alpha \|v\|$. 
It is immediate to verify that when 
 this holds, the exact OMD trajectory is $(\eta \alpha)$-balanced.
 Notably, by working with the notion of loss balance, we obtain results that are later applicable to both the stochastic and adversarial settings.

The relation between loss balance and \emph{approximate} trajectories ($\epsilon>0$) is more nuanced, since the errors we want to control naturally scale with the smoothness parameter of the objective (which is unbounded in our case).
To cope with this we introduce the notion of \emph{effective smoothness}, which roughly corresponds to the smoothness parameter associated with the line segment between the exact and approximate OMD updates.
Under the assumption that the iterates remain bounded away from zero, the effective smoothness remains finite and we may bound the 
difference between the balance of the exact and approximate OMD trajectories, as stated in our next lemma.
\begin{lemma}\label{lem:trajectory diff}
Let $\curly{w_t}_{t=1}^T,\curly{\hat{w}_t}_{t=1}^T$ be an exact trajectory and an approximate trajectory with the same $\barconst$-barrier regularizer, losses and $\eta$. Fix $0\le t_1\le t_2\le T$ and $v\in \ker(A)$ such that $\norm{v}_1 = 1$. Let $\psi>0$ be such that for every $t_1\le t\le t_2$ and $i\in[d]$ such that $v^i\ne 0$, $w_t^i\ge \psi$ and $\epsilon \le c_2\psi/2$. Then, we have:
\begin{align*}
    \hat{B}^v(t_1,t_2) \le B^v(t_1,t_2) + (t_2-t_1)\sqrt{\frac{c_2\epsilon}{\psi^\barconst}},
\end{align*}
where $\hat{B}$ is the balance of the approximate trajectory.
\end{lemma}
Next, we introduce machinery that facilitates arguments going in the other direction; namely, that a balanced trajectory remains bounded away from zero.
Our lemma below generalizes the argument given after \Cref{thm:ub_nonentropy} and bounds the first derivative of the regularizer in terms of the balance of the OMD trajectory. The bound on the first derivative may in turn be used to yield a bound on the actual iterate coordinates (such as in the special case discussed in the paragraph after \Cref{thm:ub_nonentropy}).
\begin{lemma}\label{lem:balance to gradient UB}
Let $\cK$ be a polyhedral subset of the simplex. Assume the trajectory is $k$-balanced w.r.t to the $L_1$-norm and was initialized at the uniform distribution. Then, for every $t\in[T],i\in[d]$:
\begin{align*}
    -r'(w_t^i)
    \leq 
    4kd - r'(1/2d)
    .
\end{align*}
\end{lemma}

When $\cK$ is the simplex, note that for every $i,j \in [d]$, the vector $e_i - e_j$ belongs to $\ker(A)$, where $e_i$ denotes the $i$th standard basis vector. 
Let $i^*$ denote the optimal arm (the coordinate with the smallest cumulative loss). 
We write $B^i$ to denote the balance with respect to the vector $e_{i^*} - e_i$. 
With this notation, we can state an additional lemma, relevant only when the decision space is the simplex, that bounds the iterate as a function of the balance.

\begin{lemma}\label{lem:simplex_coordinate_bounds}
Let $\cK = \Delta_d$. 
Fix a coordinate $i \in [d]$ and times $t_1, t_2 \in [T]$ such that $B^i(t_1,t_2) \le k$. 
Then:
\begin{enumerate}
    \item If $w_{t_2}^i \ge w_{t_1}^i$ then $e^{k/c_1}w_{t_2}^{i^*} \ge w_{t_1}^{i^*}$
    \item If $w_{t_2}^{i^*} \le w_{t_1}^{i^*}$ then $w_{t_2}^i \le e^{k/c_1}w_{t_1}^i$
\end{enumerate}
\end{lemma}

These results allow us to translate control over balance into control over how far the coordinates of the trajectory can drift, which will be crucial in the later proofs. The full proofs for the lemmas in this section can be found in \Cref{apx:balance}.

\subsection{Non-entropy barriers: Proof of \Cref{thm:ub_nonentropy}}\label{sec:non entropy}
We first handle the case of non-entropy barrier regularizers ($\barconst > 1$), before turning in the next subsection to the negative entropy regularizer, which requires a separate treatment. The idea is to show that barrier regularizers with $\barconst > 1$ cannot drive the iterates exponentially close to zero, and 
as a result the relevant effective smoothness parameter grows only polynomially with $T$.
Consequently, we can apply an argument similar to the regret bound in the smooth regularizer case from \Cref{thm:ub_smooth} to control the additional regret due to approximation errors. The full proof can be found in \Cref{sec:polytope upper bound}.

\begin{proofsketch}[Proof of \Cref{thm:ub_nonentropy}]
Let 
$$
    \psi \coloneqq 
    \left(\frac{c_1}{8\eta Td + c_1(2d)^{\barconst-1}}\right)^{\!1/(\barconst-1)}
    .
$$
We prove by induction on $t \in [T]$ that all coordinates remain bounded away from zero, namely $w_t^i \ge \psi$ for all $i \in [d]$. 
Assume the claim holds up to step $t-1$. 
Using \Cref{lem:not_too_far}, which bounds the step size of each iterate, we first show that $w_t^i \ge \psi/2$. 
Since the balance of an exact trajectory is always bounded by $T\eta$, \Cref{lem:trajectory diff} implies that for every normalized $v \in \ker(A)$,
\[
    B^v(1,t) \le T\eta + T\sqrt{c_2\epsilon\,(2/\psi)^{\barconst}} \le 2T\eta.
\]
Hence, the trajectory up to step $t$ is $2T\eta$-balanced. 

Applying \Cref{lem:balance to gradient UB}, we have
\begin{align*}
    -r'(w_t^i) &\le 8\eta Td - r'(1/2d), \\
    \implies\quad 
    \frac{c_1}{(w_t^i)^{\barconst-1}} &\le 8\eta Td + c_1(2d)^{\barconst-1}, \\
    \implies\quad 
    w_t^i &\ge 
    \left(\frac{c_1}{8\eta Td + c_1(2d)^{\barconst-1}}\right)^{\!1/(\barconst-1)}.
\end{align*}
The first implication follows from \Cref{lem:gradient_diff}, which provides a lower bound on the gradient difference for barrier regularizers. 
This completes the inductive step and establishes the result.
\end{proofsketch}

\subsection{Negative entropy}\label{sec:entropy}

With the negative entropy regularizer, robustness to approximation errors differs sharply between adversarial and stochastic settings. In this case, the iterates can approach zero exponentially fast, causing the effective smoothness to grow exponentially. However, under stochastic losses, the iterates do not become stuck despite approaching zero. The key lies in the balance of the loss sequence: while an arbitrary adversarial sequence can have balance as large as $T$, for i.i.d.\ stochastic losses the balance is bounded with high probability by roughly $\sqrt{T}$.

To capture this distinction, we establish two general lemmas that characterize the regret behavior as a function of the balance parameter. 
The first lemma shows that if the approximation error exceeds an exponential (in the balance) threshold, then linear regret can occur.

\begin{lemma} \label{lem:balance lower bound}
Let $d=2$, $\cK=\Delta_d$ be the simplex, and let $R(w)=\sum_{i=1}^d w_i\log w_i$ be the negative entropy regularizer.
Then, for any $\alpha \le T/2$ there exists a sequence of $\alpha$-balanced losses $\ell_1, \ldots, \ell_T \in [0,1]^d$ such that 
for any $\epsilon \ge 4\eta e^{-\eta \alpha}$, 
there exists a trajectory that is $\epsilon$-approximate w.r.t.~$(\cK, R, \ell_{1\dots T})$ and has regret $\Omega(T-2\alpha)$.
\end{lemma}

The proof idea is that the approximation errors are large enough so that the iterate coordinates may reach the region close to zero where the iterate can become ``stuck'' due to additional subsequent errors (this idea was explained thoroughly after \Cref{thm:lb_barrier_adversarial}).
This Lemma is the principle technical gradient in the proof of \Cref{thm:lb_barrier_adversarial}. Conversely, the second lemma shows that if the error is below this exponential threshold, standard regret bounds hold:

\begin{lemma}\label{lem:balance upper bound}
Let $\cK=\Delta_d$ be the simplex, let $\curly{\ell_t}_{t=1}^T$ be $\alpha$-balanced loss sequence, and let $R(w)=\sum_{i=1}^d w_i\log w_i$ be the negative entropy regularizer. Assume $\eta \le 1/16$ and $T\ge 3$,
if the approximation error satisfies 
\[
    \epsilon \le \frac{1}{d \max\curly{6e^{\eta \alpha}, 1/\eta}}\min\curly{\eta^4,1/T^2},
\]
then the regret of any $\epsilon$-approximate OMD trajectory is bounded as
\[
    \mathrm{Regret}(w) \le \frac{1}{\eta} D_R(w, w_1) + O(T \eta).
\]
\end{lemma}
\begin{proofsketch}
Let $i^*$ be the coordinate with the smallest cumulative loss. We prove the claim in two steps. 
\begin{enumerate}
    \item If the optimal arm coordinate $i^*$ of all iterates is bounded away from zero, i.e.\ $\forall t, w_t^{i^*} \ge \xi$, and in addition $\epsilon \lessapprox \eta^4\xi$, the regret bound follows (\Cref{lem:LB best UB regret}).
    \item If the losses are $\alpha$-balanced and $\epsilon \le \xi/(2T^2)$ for $\xi = 1/(d e^{\alpha\eta+1})$, then the optimal arm coordinate is bounded away from zero throughout the trajectory, i.e., $\forall t, w_t^{i^*} \ge \xi$ (\Cref{lem:balance to bounded optimal arm}).
\end{enumerate}

\paragraph{Step 1.} 
In the classical OMD analysis, first-order optimality conditions are applied at every step. The gap in these conditions depends on the effective smoothness, which in turn reflects how close the coordinates are to zero. Since some coordinates may take very small values, we apply the optimality conditions only to those with $w_t^i \ge \xi$. Using a careful argument—based on the observation that coordinates close to zero contribute little to the overall regret—we extend the proof to all coordinates. An additional challenge arises because the set of small coordinates changes over time, which we handle using the monotonicity of the Bregman divergence (see \Cref{lem:bregman monotonic}).

\paragraph{Step 2.} 
We now prove (2) by induction. 
Assume, for contradiction, that some step $t$ is the first to have $w_t^{i^*} < \xi$ 
Then there must exist a coordinate $i\ne i^*$ 
such that $w_t^i \ge 1/d$. 
We first show that for every $s < t$, $w_s^i \ge \xi$. 
Suppose not, and let $s$ be the last time for which $w_s^i < \xi$. 
We use \Cref{lem:not_too_far}, which bounds the step size of each iterate, to first establish that $w_s^i \ge \xi/2$. 
From \Cref{lem:trajectory diff},
\[
    B^i(s,t) \le \alpha\eta + T\sqrt{r''(\xi/2)\epsilon} = \alpha\eta + 1.
\]
Applying the second part of \Cref{lem:simplex_coordinate_bounds} with $(i,s,t)$ yields $w_s^i \ge \xi$, a contradiction. 
Hence, $w_s^i \ge \xi$ for all $s \le t$. 
We then use this to bound the balance from the beginning:
\[
    B^i(1,t) \le \alpha\eta + T\sqrt{r''(\xi/2)\epsilon} = \alpha\eta + 1.
\]
Since $w_t^i \ge w_1^i$, the first part of \Cref{lem:simplex_coordinate_bounds} implies that $w_t^{i^*} \ge \xi$, completing the induction.
\end{proofsketch}

Together, these two lemmas provide a clean characterization: 
linear regret is unavoidable once the error exceeds an exponential threshold in $\eta \alpha$, 
while below this threshold optimal regret guarantees are preserved. 

Before applying the lemmas, let us note that when the polytope is the simplex itself, the vectors 
$e_i - e_{i^*}$ for every $i \in [d]$ form a basis of $\ker(A)$. 
Thus, if for every $i \in [d]$ we have
\[
    \ell_{t_1:t_2}^i - \ell_{t_1:t_2}^{i^*} \le \alpha,
\]
it follows that the loss sequence is $\alpha$-balanced.

\paragraph{Implications for the main theorems.}

\begin{proof}[Proof of \Cref{thm:lb_barrier_adversarial}]
Directly by applying \Cref{lem:balance lower bound} with $\alpha=T/3$.
\end{proof}

\begin{proof}[Proof of \Cref{thm:up_barrier_adversarial}] 
Any adversarial sequence over the simplex is $T/2$-balanced: if one coordinate exceeds the best by more than $T/2$, it must actually be the best. 
Applying \Cref{lem:balance upper bound} with $\alpha=T/2$ gives the desired upper bound.
\end{proof}

\begin{proof}[Proof of \Cref{thm:ub_barrier_stochastic}] 
For i.i.d.~losses, Hoeffding's inequality and union bounds implies that with probability at least $1-\delta$, the balance is at most $\alpha =O(\sqrt{T\log(dT^2/\delta)})$, which means that 
$\eta\alpha \le \log\roundy{dT^2/\delta}$. 
Plugging this into \Cref{lem:balance upper bound} together with the fact that $D_R(w,w_1)\le \log(d)$ for all $w$ yields the stochastic upper bound.
\end{proof}

\section{Discussion}
This work provides an analysis
of how approximation errors affect Online Mirror Descent. We establish tight upper and lower regret bounds for smooth regularizers, showing that polynomially small errors suffice to maintain optimal regret. Moving beyond smoothness, we uncover a sharp separation among barrier-type regularizers: with negative entropy, exponentially small errors are necessary to avoid linear regret, whereas log-barrier and Tsallis regularizers remain robust even with polynomially large errors. We further show that while negative entropy regains robustness under stochastic losses on the full simplex, this property fails on certain polyhedral subsets. Altogether, our results reveal a fundamental sensitivity of OMD to approximation accuracy, determined jointly by the geometry of the domain, the curvature of the regularizer, and the structure of the loss sequence. Furthermore, our work provides a detailed characterization of when precision is essential and when it is not. 

A broader goal emerging from this work, left for future investigation, is to develop a comprehensive theory of inexact OMD for general regularizers and geometries. In particular, it would be valuable to characterize the robustness properties of self-concordant barrier regularizers over general convex domains.

\subsection*{Acknowledgments}
This project has received funding from the European Research Council (ERC) under the European Union’s Horizon 2020 research and innovation program (grant agreements No.\ 101078075; 882396). Views and opinions expressed are however those of the author(s) only and do not necessarily reflect those of the European Union or the European Research Council. Neither the European Union nor the granting authority can be held responsible for them. This work received additional support from the Israel Science Foundation (ISF, grant numbers 3174/23; 1357/24), and a grant from the Tel Aviv University Center for AI and Data Science (TAD). This work was partially supported by the Deutsch Foundation. OS is also partially
supported by the TAD Excellence Program for Doctoral Students in Artificial Intelligence and Data
Science from the Tel Aviv University Center for AI and Data Science (TAD) and from the Israeli
Council for Higher Education (CHE) Fellowship for Outstanding PhD Students in Data Science.

\newpage
\bibliography{cite}
\notopt{\bibliographystyle{abbrvnat}}

\newpage
\appendix

\section*{Appendix Structure}
\Cref{sec:general} provides general definitions and Lemmas used throughout the appendix. 
\Cref{apx:smooth} contains the proofs for \Cref{sec:smooth}. 
\Cref{apx:balance} introduces the notion of \emph{balance}, which is needed for the subsequent proofs. 
The remaining appendices establish the main technical arguments of the paper: 
\Cref{sec:simplex lower bounds} contains lower bounds for negative entropy (including the proof of \Cref{lem:balance lower bound} and additional results), 
\Cref{sec:simplex upper bound} contains the proof of \Cref{lem:balance upper bound}, 
\Cref{sec:polytope upper bound} contains the proof of \Cref{thm:ub_nonentropy}, 
and \Cref{sec:polytope lower bound} contains the proof of \Cref{thm:lb_nonsimplex}.

\section{General Lemmas and definitions}\label{sec:general}

\begin{definition}
We call $\gamma = (\curly{w_t}_{t=1}^T, \curly{\ell_t}_{t=1}^T, R, \eta)$ an \textbf{exact OMD trajectory} if for every $t\in[T]$:
\begin{align*}
    w_{t+1} = argmin_{w\in\Delta_d} \eta \inprod{\ell_t}{w} + D_R(w_{t-1},w)
\end{align*}
\end{definition}

\begin{definition}
We call $\gamma = (\curly{w_t}_{t=1}^T, \curly{\ell_t}_{t=1}^T, R, \eta)$ an $\epsilon$-\textbf{approximate OMD trajectory} with some $\epsilon>0$ if for every $t\in[T]$ $w_{t+1}$ is an $\epsilon$-minimizer of $\eta \inprod{\ell_t}{w} + D_R(w_{t-1},w)$.
\end{definition}

\begin{definition}
Our assumptions about the regularizers are:
\begin{itemize}
    \item There is a function $r:[0,1]\to\rr$ such that $R$ is coordinate-separated with $f_i=r$ for all $i\in[d]$
    \item $r''$ is decreasing polynomially in [0,1] and $r''(w) \ge \frac{1}{w}$ for all $w\in[0,1]$.
\end{itemize}
\end{definition}

\begin{definition}
We say that a function $F:\W\to\rr$ ($\W\subseteq\rr^d$) is coordinate-separated if there are functions $f_1,f_2,\dots,f_d$ such that $F(w) = \sum_i f_i(w_i)$ for all $w\in\W$.
\end{definition}

\begin{definition}
Let $F:\W\to\rr$ be a coordinate-separated function. Let $w^1,w^2\in\W$, we say $\beta\in\rr$ is the \textbf{effective smoothness} of $F$ w.r.t $w_1,w_2$ if for every $i\in[d]$  such that $w^1_i\ne w^2_i$ and $\alpha\in[w_i^1,w_i^2]$, we have $f_i''(\alpha) \le \beta$. 
\end{definition}

\begin{lemma}\label{lem:effective smoothness}
Let $F:\W\to\rr$ be a coordinate-separated function and Let $x_1,x_2\in\W$. If $\beta$ is the effective smoothness of $F$ w.r.t $x_1,x_2$ we have for any $w_1, w_2 \in [x_1, x_2]$: 
\begin{align*}
    F(w_1) - F(w_2) - \inprod{\nabla F(w_2)}{w_1-w_2} \le \frac{\beta}{2}\|w_1-w_2\|_2^2 \le \frac{\beta}{2}\|w_1-w_2\|_1^2
\end{align*}
\end{lemma}
\begin{proof}
The first inequality is directly from Taylor's theorem. The second is because generally $\norm{\cdot}_2 \le \norm{\cdot}_1$.
\end{proof}

\begin{lemma}\label{lem:epsilon optimality conditions}
Let $\norm{\cdot}$ be any norm, and let $f:\W\to\rr$, and let $\hat{w}, w\in\W$ where $\hat w$ is an $\epsilon$-minimizer of $f$. Assume that for all $x, y\in [w, \hat w]$ it holds that:
\begin{align*}
    f(y) - f(x) - \inprod{\nabla f(x)}{y-x} \le \frac{\beta}{2}\|y-x\|^2.
\end{align*}
Then, we have:
\begin{align*}
    \inprod{\nabla f(\hat w)}{w-\hat{w}} \ge -\max\curly{\|w-\hat{w}\|\sqrt{2\beta\epsilon},\, 2\epsilon}
\end{align*}
Additionally, let $D = \max_{w',w''\in\cK}\|w'-w''\|$ and assume $\epsilon \le \frac{D^2\beta}{2}$. We have:
\begin{align*}
    \inprod{\nabla f(\hat w)}{w-\hat{w}} \ge -D\sqrt{2\beta\epsilon}
\end{align*}
We note that this holds for coordinate-separated function with effective smoothness $\beta$ (with $\ell_1$ or $\ell_2$ norm, see \Cref{lem:effective smoothness}) or any general $\beta$-smooth function.
\end{lemma}

\begin{proof}
From the assumptions of the Lemma, for any $\gamma \in [0,1]$:
\begin{align*}
    f(\hat{w} + \gamma(w-\hat{w})) &\le f(\hat{w}) + \gamma\nabla f(\hat{w})(w-\hat{w}) + \gamma^2\frac{\beta}{2}\|w-\hat{w}\|^2  \\
    \nabla f(\hat{w})(w-\hat{w}) &\ge \frac{1}{\gamma}\roundy{f(\hat{w} + \gamma(w-\hat{w})) - f(\hat{w})} - \gamma\frac{\beta}{2}\|w-\hat{w}\|^2 \\
    &\ge -\roundy{\frac{\epsilon}{\gamma} + \gamma\frac{\beta}{2}\|w-\hat{w}\|^2} \\
\end{align*}
Notice that if $2\epsilon \ge \|w-\hat{w}\|\sqrt{2\beta\epsilon}$, we have $\epsilon \ge \frac{\beta}{2}\|w-\hat{w}\|^2$ thus for $\gamma = 1$:
\begin{align*}
    \nabla f(\hat{w})(w-\hat{w}) \ge -\roundy{\epsilon + \frac{\beta}{2}\|w-\hat{w}\|^2} \ge -2\epsilon
\end{align*}
Else, for $\gamma = \frac{\sqrt{2\epsilon}}{\sqrt{\beta}\|w-\hat{w}\|} \le 1$:
\begin{align*}
    \nabla f(\hat{w})(w-\hat{w}) \ge \|w-\hat{w}\|\sqrt{2\beta\epsilon}
\end{align*}

If $\epsilon \le \frac{D^2\beta}{2}$, we have:
\begin{align*}
    \sqrt{\epsilon} &\le \frac{D\sqrt{\beta}}{\sqrt{2}}\\
    &= \frac{D\sqrt{2\beta}}{2}\\
    \Leftrightarrow 2\epsilon &\le D\sqrt{2\beta\epsilon} \\
    \Rightarrow \inprod{\nabla f(\hat w)}{w-\hat{w}} &\ge -D\sqrt{2\beta\epsilon}
\end{align*}
\end{proof}

\begin{lemma}\label{lem:solve quadratic}
If for some $a,b,c > 0$ we have $ax^2 - bx - c \le 0$, then $x < \frac{b}{a} + \sqrt{\frac{c}{a}}$
\end{lemma}
\begin{proof}
Assume $x = \frac{b}{a} + \sqrt{\frac{c}{a}}$, we have:
\begin{align*}
    ax^2 - bx - c = \frac{b^2}{a} + 2b\sqrt{\frac{c}{a}} + c - \frac{b^2}{a} - b\sqrt{\frac{c}{a}} - c = b\sqrt{\frac{c}{a}} > 0
\end{align*}
The minimum point of the parabola is at $x = \frac{b}{2a}$, so it only increases for $x > \frac{b}{a} + \sqrt{\frac{c}{a}}$.
\end{proof}

\begin{lemma}\label{lem:max step}
Let $\gamma = (\curly{w_t}_{t=1}^T, \curly{\ell_t}_{t=1}^T, R, \eta)$ be an $\epsilon$-approximate trajectory above the simplex with $\eta \le \frac{1}{4}$ and coordinate-separable regularizer. Let $h=\min\curly{r''(w_t^i), r''(w_{t+1}^i)}$. Then for any $i\in[d]$:
\begin{align*}
    \abs{w_t^i - w_{t+1}^i} < \frac{4\eta}{h} + \sqrt{\frac{\epsilon}{h}}
\end{align*}
\end{lemma}
\begin{proof}
Fix $i\in[d]$. We will prove for $w_{t+1}^i \le w_t^i$. The proof for the other direction is identical.

Let $i_1, \ldots, i_m$ be an arbitrary set of coordinates that satisfies the following. For $S\eqq \{i_1, \ldots, i_{m-1}\}$, $i' \eqq i_m$ it holds that:
\begin{align}
    \forall j\in \roundy{S\cup i'} \quad w_{t+1}^j &\ge w_t^j \label{eq:S bugger}\\
    \sum_{j\in S}w_{t+1}^j - w_t^j &< w_t^i - w_{t+1}^i \label{eq:S isnt enough}\\
    \sum_{j\in \roundy{S\cup i'}}w_{t+1}^j - w_t^j &\ge w_t^i - w_{t+1}^i \label{eq:S with i'}
\end{align}
Namely, $S\cup i'$ is a set of coordinates that were increased in this step. The total increase of all the coordinates in $S$ is less than the decreased in $i$, but with the increase of $i'$ it is more than the decrease of $i$. Such coordinates exist since the difference that the $i$th coordinate was moved downward there must be a set of coordinates that upward to keep that sum of coordinate $1$.

Denote $\tilde{w}$ such that:
\begin{align*}
    \forall j\in S \quad \tilde{w}^j &= w_t^j \\
    \tilde{w}^i &= w_t^i\\
    \tilde{w}^{i'} &= w_{t+1}^{i'} + \sum_{j\in \roundy{S\cup i}}w_{t+1}^j - w_t^j \\
    \text{o.w} \quad \tilde{w}^j &= w_{t+1}^j
\end{align*}
From \Cref{eq:S isnt enough} we have that $\tilde{w}^{i'} < w_{t+1}^{i'}$. From \Cref{eq:S with i'} we have that $\tilde{w}^{i'} \ge w_{t}^{i'}$.

$\tilde{w}$ is a probability since all of its coordinates are $\ge 0$ and:
\begin{align*}
    \sum_{j\in[d]} \tilde{w}^j &= \sum_{j\in (S \cup i)}\tilde{w}^j + \sum_{j\notin (S \cup \curly{i,i'})}\tilde{w}^j + \tilde{w}^{i'}\\
    &= \sum_{j\in (S \cup i)}{w}_{t}^j + \sum_{j\notin (S \cup \curly{i,i'})}{w}_{t+1}^j \sum_{j\in \roundy{S\cup i}}w_{t+1}^j - w_t^j\\
    &= \sum_{j\in[d]} {w}_{t+1}^j \\
    &= 1
\end{align*}

Since for all $j\in S$ we have $\tilde{w}^j = w_t^j$, we have:
\begin{align*}
    \sum_{j\in S} D_r(w_t^j,\tilde{w}^j) = 0 \le \sum_{j\in S} D_r(w_t^j,w_{t+1}^j)
\end{align*}

From Taylor inequality and the definition of $h$:
\begin{align*}
    D_r(\tilde{w}^i, w_t^i) &= 0 \\
    D_r(w_t^i, w_{t+1}^i) &\ge \frac{h}{2}(w_{t+1}^i - w_t^i)^2  \\
    D_r(w_t^i, w_{t+1}^i) &\ge D_r(\tilde{w}^i, w_t^i) + \frac{h}{2}(w_{t+1}^i - w_t^i)^2
\end{align*}

Since $w_t^{i'} \le \tilde{w}^{i'} < w_{t+1}^{i'}$ we have $D_r(w_t^{i'}, \tilde{w}^{i'}) < D_r(w_t^{i'}, w_{t+1}^{i'})$.

Since $\tilde{w}^j = w_{t+1}^j$, we have $\sum_{j\notin (S\cup\curly{i,i'})} D_r(w_t^j,\tilde{w}^j) = \sum_{j\notin (S\cup\curly{i,i'})} D_r(w_t^j,w_{t+1}^j)$. 

Summing all we have:
\begin{align*}
    D_R(w_t, w_{t+1}) - D_R(w_t, \tilde{w}) \ge \frac{h}{2}(w_{t+1}^i - w_t^i)^2
\end{align*}

From the definition of $\tilde{w}_i$ we have $\|\tilde{w} - w_{t+1}\| = 2(w_{t+1}^i - w_t^i)$. Thus, from Holder:
\begin{align*}
    \eta\inprod{\ell_t}{w_{t+1} - \tilde{w}} \ge -2\eta\abs{w_{t+1}^i - w_t^i}
\end{align*}

Since $w_{t+1}$ is an $\epsilon$-minimizer of the OMD objective:
\begin{align*}
    \epsilon &\ge \eta\inprod{\ell_t}{w_{t+1} - \tilde{w}} + D_R(w_t, w_{t+1}) - D_R(w_t, \tilde{w})\\
   &\ge \frac{h}{2}(w_{t+1}^i - w_t^i)^2 -2\eta\abs{w_{t+1}^i - w_t^i}
\end{align*}

From \Cref{lem:solve quadratic} we get:
\begin{align*}
    \abs{w_t^i - w_{t+1}^i} < \frac{4\eta}{h} + \sqrt{\frac{\epsilon}{h}}
\end{align*}
\end{proof}

\begin{lemma}\label{lem:not_too_far}
Let $\gamma = (\curly{w_t}_{t=1}^T, \curly{\ell_t}_{t=1}^T, R, \eta)$ be an $\epsilon$-approximate trajectory above the simplex with $\eta \le \frac{1}{16c_1}$ and $\barconst$-barrier regularizer. Let $t\in[T]$ and $i\in[d]$ be such that $\epsilon \le \frac{(w_t^i)^\barconst}{16c_1}$, then:
\begin{align*}
    w_{t-1}^i &\ge \frac{1}{2}w_t^i\\
    w_{t+1}^i &\ge \frac{1}{2}w_t^i
\end{align*}
\end{lemma}
\begin{proof}
We will prove for $w_{t-1}^i$ but the same proof goes for $t+1$. The interesting case is obviously $w_{t-1}^i < w_t^i$, so continuing assuming that.

We have:
\begin{align*}
    \min\curly{r''(w_t^i), r''(w_{t-1}^i)} \ge \frac{c_1}{\max\curly{w_t^i, w_{t-1}^i}^\barconst} = \frac{c_1}{(w_t^i)^\barconst} \ge \frac{c_1}{w_t^i}
\end{align*}

From \Cref{lem:max step}:
\begin{align}\label{eq:single_step_relation}
    w_{t-1}^i \ge w_t^i - 4c_1\eta w_t^i - \sqrt{c_1\epsilon w_t^i}
\end{align}

Since $\eta \le 1/16c_1$:
\begin{align}\label{eq:first_term}
    4c_1\eta w_t^i \le \frac{w_t^i}{4}
\end{align}

From the assumption on $\epsilon$ and the fact that $r''(w_t^i) \ge w_t^i$:
\begin{align}\label{eq:sec_term}
    \sqrt{c_1\epsilon w_t^i} \le \sqrt{\frac{(w_t^i)^2}{16}} = \frac{w_t^i}{4}
\end{align}

Placing \Cref{eq:first_term,eq:sec_term} in \Cref{eq:single_step_relation} gives the desired results.
\end{proof}

\begin{lemma}[Three-points identity]\label{lem:three_points}
For every differentiable function $R$:
\begin{align*}
\forall ~x,y,z: \quad
\big (\nabla R(z)-\nabla R(y) \big) \cdot (y-x)
=
D_R(x,z) - D_R(x,y) - D_R(y,z)
\end{align*}
\end{lemma}
\begin{proof}
\begin{align*}
D_R(x,z) - D_R(x,y) - D_R(y,z)
&=
R(x) - R(z) - \nabla R(z) \cdot (x-z)
\\
&- R(x) + R(y) + \nabla R(y) \cdot (x-y)
\\
&- R(y) + R(z) + \nabla R(z) \cdot (y-z)
\\
&=
\big (\nabla R(z)-\nabla R(y) \big) \cdot (y-x)
\end{align*}
\end{proof}

\begin{lemma}[OMD Helper]\label{lem:omd helper}
\begin{align*}
\ell_t \cdot (w_t-w_{t+1}) - \frac1\eta D_R(w_{t+1},w_t)
\leq 
\frac\eta2\|\ell_t\|_*^2
\end{align*}
\end{lemma}
\begin{proof}
From the strong convexity of $R$:
\begin{align*}
\frac1\eta D_R(w_{t+1},w_t)
\geq 
\frac{1}{2\eta} \|w_{t+1}-w_t\|^2
\end{align*}
By Holder:
\begin{align*}
\ell_t \cdot (w_t-w_{t+1}) 
&\leq \|w_t-w_{t+1}\| \, \|\ell_t\|_*
\\
&\leq \frac{1}{2\eta} \|w_t-w_{t+1}\|^2 + \frac{\eta}{2} \|\ell_t\|_*^2\\
\end{align*}
We used the fact that $ab \le \frac{1}{2}a^2 + \frac{1}{2}b^2$ for every $a,b\ge 0$. 
\end{proof}

\section{Smooth Regularizer}\label{apx:smooth}

\begin{theorem*}[Restatement of \Cref{thm:ub_smooth}]\label{thm:smooth upper bound}
Let $\gamma = (\curly{w_t}_{t=1}^T, \curly{\ell_t}_{t=1}^T, R, \eta)$ be an $\epsilon$-approximate trajectory above a convex set such that $R$ is $\beta$-smooth, and let $D$ be the diameter of the domain. Assume $\epsilon \le D^2/2$, then for any $w\in\cK$:
\begin{align*}
    \reg(w) \le O\roundy{\frac{1}{\eta}D_R(w, w_1) + T\eta + \frac{TD\sqrt{\beta\epsilon}}{\eta}}
\end{align*}
\end{theorem*}
\begin{proof}
From the strong convexity of $R$ we have that $\beta \ge 1$, which means that from the assumptions $\epsilon \le D^2\beta/2$. Then, from \Cref{lem:epsilon optimality conditions}, for every $t$:
\begin{align*}
    \inprod{\eta\ell_t + \nabla R(w_{t+1}) - \nabla R(w_t)}{w^* - w_{t+1}} \ge -D\sqrt{2\beta\epsilon}
\end{align*}

From here it is straightforward standard OMD arguments:
\begin{align*}
\eta \ell_t \cdot (w_{t+1}-w^*) 
&\leq (\nabla R(w_{t+1}) - \nabla R(w_{t})) \cdot (w^*-w_{t+1}) + D\sqrt{2\beta\epsilon}
\\
&=
D_R(w^*,w_t) - D_R(w^*,w_{t+1}) - D_R(w_{t+1},w_t) + D\sqrt{2\beta\epsilon} \\
\end{align*}
Summing for all $t\in[T]$:
\begin{align*}
    \sum_{t=1}^T \ell_t \cdot (w_{t+1}-w^*)
&\leq
\frac1\eta D_R(w^*,w_1) - \frac1\eta \sum_{t=1}^T D_R(w_{t+1},w_t) + \frac{TD\sqrt{2\beta\epsilon}}{\eta}
\end{align*}
From \Cref{lem:omd helper}:
\begin{align*}
    \reg(w^*) &\le O\roundy{\frac{1}{\eta}D_R(w^*, w_1) + T\eta + \frac{TD\sqrt{\beta\epsilon}}{\eta}}
\end{align*}
\end{proof}

\begin{theorem*}[Restatement of \Cref{thm:lb_smooth}]\label{thm:smooth lower bound}
For every $\beta,\epsilon$, there is an OMD $\epsilon$-approximate trajectory $\gamma = (\curly{w_t}_{t=1}^T, \curly{\ell_t}_{t=1}^T, R, \eta)$ above a convex set with diameter $D$ with $R$ being $\beta$-smooth and constant losses ($\ell_t=\ell$ for some $\ell$ for all $t\in[T]$) that achieves a regret of 
\begin{align*}
    \Omega\roundy{\min\roundy{\frac{TD\sqrt{\beta\epsilon}}{\eta}, DT}}
\end{align*}
\end{theorem*}
\begin{proof}
Consider the domain $[0,D]$, $w_1 = \frac D2$. The regularizer is $R(w) = \frac{\beta}{2}w^2$. The loss is $\ell = \min\curly{\frac{\sqrt{2\beta\epsilon}}{\eta},1}$.

We will now show by induction that $w_t = w_1$ for all $t$ is a valid $\epsilon$-approximate trajectory. This trajectory suffers a loss of $\Theta\roundy{\min\roundy{\frac{TD\sqrt{\beta\epsilon}}{\eta}, DT}}$, which means a same regret comparing to $w^* = 0$.

Assume true for $t-1$, we will prove for $t$. 

We start by finding the optimal $w_t^*$ (the optimal solution for $\phi_t$) by differentiating and comparing to $0$:
\begin{align*}
    \eta\ell + \beta(w_t^* - w_{t-1}) &= 0 \\
    \Longleftrightarrow\quad w_t^* &= w_{t-1} - \frac{\eta}{\beta}\ell
\end{align*}
Placing it in the objective function:
\begin{align*}
    \eta\ell w_t^* + \frac{\beta}{2}\roundy{w_t^* - w_{t-1}}^2 &= \eta\ell\roundy{w_{t-1} - \frac{\eta}{\beta}\ell} + \frac{\beta}{2}\roundy{w_{t-1} - \roundy{w_{t-1} - \frac{\eta}{\beta}\ell}}^2 \\
    &= \eta\ell w_{t-1} - \frac{\eta^2\ell^2}{\beta} + \frac{\eta^2\ell^2}{2\beta} \\
    &= \eta\ell w_{t-1} - \frac{\eta^2\ell^2}{2\beta}
\end{align*}
Which means that the difference in the objective function between $w_t^*$ and $w_{t-1}$ is $\frac{\eta^2\ell^2}{2\beta}$.

From the definition of $\ell$:
\begin{align*}
    \ell &\le \frac{\sqrt{2\beta\epsilon}}{\eta} \\
    \Longleftrightarrow\quad\frac{\eta^2\ell^2}{2\beta} &\le \epsilon
\end{align*}
Which means that $w_{t-1}$ is an $\epsilon$-minimizer.
\end{proof}  

\section{Balance}\label{apx:balance}
All the lemmas in this section assumes $\nu$-barrier regularizer.
\subsection{General}

\begin{definition}
Assume $\cK$ is a polytope defined in standard form $\curly{w \in \rr^d: Aw=b\land \roundy{w_i\ge 0,\,\forall i\in[d]}}$. For every $v\in \ker(A)$, denote the balance of an OMD trajectory w.r.t $v$:
\begin{align*}
    B^v_\gamma(t_1,t_2) = \inprod{\nabla R(w_{t_1}) - \nabla R(w_{t_2})}{v}
\end{align*}
Additionally, if for every $v\in \ker(A)$ such that $\|v\|\le 1$ and $t_1,t_2$ we have $B^v_\gamma(t_1,t_2) \le k$, we say the trajectory is $k$ balanced.
\end{definition}

\begin{lemma} \label{lem:balance trasitive}
\begin{align*}
    B_\gamma^i(t_1,t_2) + B_\gamma^i(t_2, t_3) = B_\gamma^i(t_1,t_3)
\end{align*}
\end{lemma}
\begin{proof}
\begin{align*}
    \inprod{\nabla R(w_{t_1}) - \nabla R(w_{t_2})}{v} + \inprod{\nabla R(w_{t_2}) - \nabla R(w_{t_3})}{v} = \inprod{\nabla R(w_{t_1}) - \nabla R(w_{t_3})}{v}
\end{align*}
\end{proof}

\begin{lemma}\label{lem:polytope optimality}
Assume $\cK$ is a polytope. For some differentiable function $f:\cK\to\rr$, let $w^*$ be the minimizer of $f$ such that for all $i\in[d]$, $(w^*)^i > 0$ 
. For every $w\in\cK$ we have:
\begin{align*}
    \inprod{\nabla f(w^*)}{w-w^*} = 0
\end{align*}
\end{lemma}
\begin{proof}
Denote $v=w-w^*$. Since $min_i \,\hat{w}^i > 0$, and $v\in \ker(A)$ where $A$ is the matrix of the polytope $\cK$, there is an $\alpha$ such that both $w^* + \alpha v \in \cK$ and $w - \alpha v \in \Delta_d$.

Since $w^*$ is a minimizer, from first order optimality conditions: 
\begin{align*}
    \inprod{\nabla f(\hat{w})}{w^*+\alpha v - w^*} &\ge 0 \\
    \inprod{\nabla f(\hat{w})}{w^*-\alpha v - w^*} &\ge 0
\end{align*}
Which means that:
\begin{align*}
        \inprod{\nabla f(\hat{w})}{v} &\ge 0 \\
    \inprod{\nabla f(\hat{w})}{-v} &\ge 0
\end{align*}
Which is our desired results.
\end{proof}

\begin{lemma*}[Restatement of \Cref{lem:polytope balance}] 
Let $\gamma=(\curly{w_t}_{t=1}^T, \curly{\ell_t}_{t=1}^T, R, \eta)$ be an exact OMD trajectory. For every $v\in \ker(A)$ and times $t_1,t_2$:
\begin{align*}
    B^v(t_1,t_2) = \eta\inprod{\ell_{t_1:t_2}}{v}
\end{align*}
\end{lemma*}
\begin{proof}
Fix some $t'\in[t_1,t_2]$. There is some small $\alpha$ such that both $w_{t'}+\alpha v$ and $w_{t'}-\alpha v$ is in the polytope 
. From \Cref{lem:polytope optimality} ($w_{t'}^i>0$ since the regularizer is undefined in $0$):
\begin{align*}
    \inprod{\ell_{t'-1} + \nabla R(w_{t'}) - \nabla R(w_{t'-1})}{v} = 0
\end{align*}

Summing for all $t'\in[t_1,t_2]$ gives the desired results.
\end{proof}

\begin{lemma*}[Restatement of \Cref{lem:trajectory diff}]
Let $\gamma=(\curly{w_t}_{t=1}^T, \curly{\ell_t}_{t=1}^T, R, \eta)$ and $\hat{\gamma}=(\curly{\hat{w}_t}_{t=1}^T, \curly{\ell_t}_{t=1}^T, R, \eta)$ be an exact OMD trajectory and $\epsilon$-approximate OMD trajectory.

Let $0\le t_1\le t_2 \le T$, $v\in \ker(A)$ such that $\|v\|=1$ and let $\psi>0$ be such that for every $t_1\le t \le t_2$, for all $i\in[d]$ such that $v^i\ne 0$, $\hat{w}_t^i \ge \psi$. We also assume that $\epsilon \le c_2\psi/2$. we have:
\begin{align*}
    B_{\hat{\gamma}}^v(t_1,t_2) \le B_{\gamma}^v(t_1,t_2)  + (t_2-t_1)\sqrt{\frac{c_2\epsilon}{\psi^\barconst}}
\end{align*}
\end{lemma*}
\begin{proof}
We will prove it using induction for $t\in[t_1,t_2]$. The base $t=t_1$ is trivial.

Assume true for $t-1$, namely:
\begin{align*}
    B_{\hat{\gamma}}^v(t_1,t-1) \le B_{\gamma}^v(t_1,t-1)  + (t-1-t_1)\sqrt{\frac{c_2\epsilon}{\psi^\barconst}}
\end{align*}

From the assumptions of the lemma we have $\hat{w}_t+\psi v \in \cK$. Additionally, the effective smoothness is $\frac{c_2}{\psi^\barconst}$. From \Cref{lem:epsilon optimality conditions}, since $\hat{w}_t$ is an $\epsilon$-minimizer of $\phi_t$:
\begin{align*}
    \inprod{\eta\ell_t + \nabla R(\hat{w}_t) - \nabla R(\hat{w}_{t-1})}{\psi v} \ge -\max\curly{\psi\sqrt{\frac{2c_2\epsilon}{\psi^\barconst}}, \,2\epsilon}
\end{align*}

Since $\epsilon\le c_2\psi/2$ and from the definition of barrier regularizer:
\begin{align*}
    \psi\sqrt{\frac{2c_2\epsilon}{\psi^\barconst}} &\ge \psi\sqrt{\frac{2c_2\epsilon}{\psi}}\\
    &= \sqrt{2c_2\epsilon\psi}\\
    &\ge 2\epsilon
\end{align*}

Which means:
\begin{align*}
\inprod{\eta\ell_t + \nabla R(\hat{w}_t) - \nabla R(\hat{w}_{t-1})}{\psi v} &\ge -\psi\sqrt{\frac{2c_2\epsilon}{\psi^\barconst}}
\end{align*}
Dividing by $\psi>0$:
\begin{align*}
    -\sqrt{\frac{2c_2\epsilon}{\psi^\barconst}} &\le \inprod{\eta\ell_t}{v} - B_{\hat{\gamma}}^v(t-1,t) \\
    &= B_\gamma^v(t-1,t) - B_{\hat{\gamma}}^v(t-1,t) \tag{\Cref{lem:polytope balance}}
\end{align*}
Adding the induction assumption:
\begin{align*}
   B_{\hat{\gamma}}^v(t_1,t-1) + B_{\hat{\gamma}}^v(t-1,t) &\le B_\gamma^v(t_1,t-1) + B_\gamma^v(t-1,t) + \sqrt{\frac{2c_2\epsilon}{\psi^\barconst}} + (t-1-t_1)\sqrt{\frac{c_2\epsilon}{\psi^\barconst}}\\
   \implies B_{\hat{\gamma}}^v(t_1,t) &\le B_\gamma^v(t_1,t) + (t-t_1)\sqrt{\frac{c_2\epsilon}{\psi^\barconst}}
\end{align*}
The last is from \Cref{lem:balance trasitive}.
\end{proof}

\subsection{Simplex subset}
The lemmas in this section assumes that the polytope is a subset of the simplex. That is, for every $w$ such that $Aw=b$, $\|w\|_1=1$. Additionally, the primal norm is assumed to be $L_1$ norm.

\begin{lemma} \label{lem:sum elem ker}
Let $v$ be a vector in the kernel of $A$. The sum of the elements of $v$ is $0$.
\end{lemma}
\begin{proof}
Denote $w=w_1+\frac{1}{d\|v\|_{\infty}}v$. It is in the polytope - all the elements of $\frac{1}{d\|v\|_{\infty}}v$ are smaller then $1/d$ and thus the all the elements of $w$ greater than $0$, and since $v$ is in the kernel of $A$ we have:
\begin{align*}
    Aw = Aw_1+A\frac{1}{d\|v\|_{\infty}}v = Aw_1 = b
\end{align*}
Thus, we have $\|w\|=1$. Since also $\|w_1\|=1$:
\begin{align*}
    \frac{1}{d\|v\|_{\infty}}\sum_{i=1}^dv^i = \sum_{i=1}^dw^i - \sum_{i=1}^dw_1^i = 1 - 1 = 0
\end{align*}
\end{proof}

\begin{lemma*}[Restatement of \Cref{lem:balance to gradient UB}]
Let $\gamma=(\curly{w_t}_{t=1}^T, \curly{\ell_t}_{t=1}^T, R, \eta)$ be a $k$-balanced OMD trajectory with $w_1 = (1/d,1/d,\dots 1/d)$ and coordinate-separated regularizer. For every $t\in[T],i\in[d]$:
\begin{align*}
    -r'(w_t^i) \le \max\curly{4kd - r'(1/d), -r'(1/2d)}
\end{align*}
\end{lemma*}
\begin{proof}
Since $w_1$ is uniform and \Cref{lem:sum elem ker}, for every $v\in \ker(A)$:
\begin{align*}
    \inprod{\nabla R(w_1)}{v} = r'(1/d)\sum_{i=1}^dv^i = 0
\end{align*}

Let $v=w_1-w_t\in \ker(A)$. Notice that since $\|w_t\|_1=\|w_1\|_1=1$, from triangle inequality $\|v\|\le 2$. From \Cref{lem:polytope balance}:
\begin{align*}
    2k &\ge \inprod{-\nabla R(w_t)}{v}\\
    &=-\sum_{i=1}^d r'(w_t^i)v^i  \\
    &= -\sum_{i:v^i > 0}^d r'(w_t^i)v^i - \sum_{i:v^i \le 0}^d r'(w_t^i)v^i  \\
\end{align*}

Denote:
\begin{align*}
    \sum_{i:v^i > 0}^d v^i = \alpha
\end{align*}

From \Cref{lem:sum elem ker}:
\begin{align*}
    -\sum_{i:v^i \le 0}^d v^i = \alpha
\end{align*}

If $v^i \le 0$ it means that $w_t^i \ge w_1^i = 1/d$, thus:
\begin{align*}
    -\sum_{i:v^i \le 0}^d r'(w_t^i)v^i &\ge -r'(1/d)\sum_{i:v^i \le0}^d v^i \ge r'(1/d)\alpha
\end{align*}
We used the fact that from the convexity of $r$, $r'$ is monotonically increasing (as $r''\ge 0$).

Denote $\bar{i}=\arg\min_{i\in[d]}w_t^i$, we have:
\begin{align*}
    2k - \alpha r'(1/d) &\ge -r'(w_t^{\bar{i}})v^{\bar{i}} - \sum_{i:v^i > 0,i\ne\bar{i}}^d r'(w_t^i)v^i\\
    &\ge -r'(w_t^{\bar{i}})v^{\bar{i}} - r'(1/d)\sum_{i:v^i > 0,i\ne\bar{i}}^d v^i\\
    &= -r'(w_t^{\bar{i}})v^{\bar{i}} - r'(1/d)(\alpha-v^{\bar{i}})\\
\end{align*}
Subtracting from both sides:
\begin{align*}
    2k - r'(1/d)v^{\bar{i}} &\ge -r'(w_t^{\bar{i}})v^{\bar{i}}
\end{align*}

If $v^{\bar{i}} \le 1/2d$ we have $w_t^{\bar{i}} \ge 1/2d$. Since $r'$ is monotonically increasing, this means that for all $i\in[d]$ $r'(w_t^i) \ge r'(1/2d)$ which concludes the proof. Else, dividing by $v^{\bar{i}} \ge 1/2d$:
\begin{align*}
    -r'(w_t^i) \le 4kd - r'(1/d)
\end{align*}
\end{proof}

\subsection{Simplex}
\begin{definition}\label{def:balance_simplex}
We denote the balance of an OMD trajectory above the simplex $\gamma=(\curly{w_t}_{t=1}^T, \curly{\ell_t}_{t=1}^T, R, \eta)$ w.r.t to a coordinate $i$ and $0\le t_1\le t_2 \le T$ to be:
\begin{align*}
    B_\gamma^i(t_1,t_2) = r'(w_{t_1}^{i^*}) - r'(w_{t_2}^{i^*}) + r'(w_{t_2}^i) - r'(w_{t_1}^i)
\end{align*}
We say that an OMD trajectory is $k$-balanced if, for every $0\le t_1\le t_2 \le T$ and coordinate $i$:
\begin{align*}
    B_\gamma^i(t_1,t_2) \le k
\end{align*}
One can notice that it is a private case for the general polytope definition.
\end{definition}

\begin{lemma*}[Restatement of \Cref{lem:simplex_coordinate_bounds}]\label{lem:balance to frac}
Let $\gamma=(\curly{w_t}_{t=1}^T, \curly{\ell_t}_{t=1}^T, R, \eta)$ be an approximate OMD trajectory. Fix $t_1,t_2\in[T]$ and $i\in [d]$ such that $B_\gamma^i(t_1,t_2) \le k$. Then:
\begin{enumerate}
    \item If $w_{t_2}^i \ge w_{t_1}^i$ then $e^{k/c_1}w_{t_2}^{i^*} \ge w_{t_1}^{i^*}$
    \item If $w_{t_2}^{i^*} \le w_{t_1}^{i^*}$ then $w_{t_2}^i \le e^{k/c_1}w_{t_1}^i$
\end{enumerate}
\end{lemma*}
\begin{proof}
We will prove the first statement and the second follows in just the same way.

Assume by contradiction that $e^kw_{t_2}^{i^*} < w_{t_1}^{i^*}$. Since $w_{t_2}^i \ge w_{t_2}^i$ we have $r'(w_{t_2}^i) \ge r'(w_{t_1}^i)$, which means:
\begin{align*}
    k &\ge B_\gamma^i(t_1,t_2)\\
    &= r'(w_{t_1}^{i^*}) - r'(w_{t_2}^{i^*}) + r'(w_{t_2}^i) - r'(w_{t_1}^i)\\
    &\ge r'(w_{t_1}^{i^*}) - r'(w_{t_2}^{i^*}) \\
    &= \int_{w_{t_2}^{i^*}}^{w_{t_1}^{i^*}}r''(w)dw \\
    &> \int_{w_{t_2}^{i^*}}^{e^{k/c_1}w_{t_2}^{i^*}}r''(w)dw \tag{$r''(w) > 0$}\\
    &\ge \int_{w_{t_2}^{i^*}}^{e^{k/c_1}w_{t_2}^{i^*}}\frac{c_1}{w}dw\\
    &= c_1\roundy{\log\roundy{e^{k/c_1}w_{t_2}^{i^*}} - \log\roundy{w_{t_2}^{i^*}}}\\
    &= k
\end{align*}
Which is a contradiction $k>k$.
\end{proof}

\begin{lemma}\label{lem:simplex_trajectory_diff}
Let $\gamma=(\curly{w_t}_{t=1}^T, \curly{\ell_t}_{t=1}^T, R, \eta)$ and $\hat{\gamma}=(\curly{\hat{w}_t}_{t=1}^T, \curly{\ell_t}_{t=1}^T, R, \eta)$ be an optimal OMD trajectory and $\epsilon$-approximate OMD trajectory.

Let $0\le t_1\le t_2 \le T$, $i\in [d]$ and $\psi>0$ be such that for every $t_1\le t \le t_2$, $\hat{w}_t^i \ge \psi$ and $\hat{w}_t^{i^*} \ge \psi$. We also assume that $\epsilon \le \psi/2$. we have:
\begin{align*}
    B_{\hat{\gamma}}^i(t_1,t_2) \le B_{\gamma}^i(t_1,t_2)  + (t_2-t_1)\sqrt{r''(\psi)\epsilon}
\end{align*}
\end{lemma}
\begin{proof}
It is direct consequence of \Cref{lem:trajectory diff} for the case of $v_i = e_{i^*} - e_i$ ($e_j$ is the $j$th element of the standard basis).
\end{proof}

\section{Lower bounds for negative entropy}\label{sec:simplex lower bounds}
\begin{lemma}\label{lem:stuck}
Let $\gamma = (\curly{w_t}_{t=1}^T, \curly{\ell_t}_{t=1}^T, R, \eta)$ be an $\epsilon$-approximate trajectory with $d=2$ and $\barconst$-barrier regularizer. If for some coordinate $i$ there is $\tau\in[T]$ such that $\frac{4\eta}{c_1} \roundy{w_\tau^i}^\barconst \le \epsilon$, then for any possible losses for $t\ge \tau$, having $w_{t}^i = w_\tau^i$ makes a valid error trajectory.
\end{lemma}
\begin{proof}
We'll prove by induction. Assume true for $w_t^i$, we'll prove for $w_{t+1}^i$.

Denote $\tilde{w}_{t+1}$ such that:
\begin{align*}
    \tilde{w}_{t+1} &= arg\min_{w\in\Delta_2}\phi_t(w)
\end{align*}
From \Cref{lem:max step} with $\epsilon = 0$ we get:
\begin{align*}
    \abs{w_t^i-\tilde{w}_{t+1}^i} &\le \frac{4\eta}{r''(w_t^i)} \le \frac{4\eta}{c_1} \roundy{w_t^i}^\barconst \le \epsilon\\
\end{align*}
Thus:
\begin{align*}
    \inprod{\ell_t}{\tilde{w}_{t+1}-w_t} \le \epsilon
\end{align*}
Since by definition $D_R(w_t,w_t) \le D_R(\tilde{w}_{t+1},w_t)$, we get:
\begin{align*}
    \phi(w_t) \le \phi(\tilde{w}_{t+1}) + \epsilon
\end{align*}
Which means that $w_t$ is an $\epsilon$-minimizer, as needed.
\end{proof}

\begin{lemma*}[Restatement of \Cref{lem:balance lower bound}]\label{lem:balance lower bound apx}
Assume for some $\alpha \ge T/2$, $\frac{1}{\eta}\log \roundy{\frac{4\eta}{\epsilon}} \le \alpha$ with negative entropy regularizer, there is an instance above the simplex with $\alpha$-balanced losses that the regret achieved is $\Omega\roundy{T-2\alpha}$.
\end{lemma*}
\begin{proof}
We construct an instance with $d=2$ and $(1,0)$ losses for the first $\tau = \frac{1}{\eta} \log \roundy{\frac{4\eta}{\epsilon}}$ and then $(0,1)$. Since $\tau < T/2$ we have that the optimal coordinate is $1$. We have:
\begin{align*}
    w_\tau^1 \le e^{-\eta\tau} = \frac{\epsilon}{4\eta}
\end{align*}
From \Cref{lem:stuck}, it is a valid error trajectory if for every $t\ge \tau$, $w_t^1 \le \frac{\epsilon}{4\eta} \le \frac{1}{2}$. Thus, the regret for those steps is $\Omega(T-\tau)$. Adding the first $\tau$ steps we get a regret bound of $\Omega(T-2\tau) \ge \Omega(T-2\alpha)$.
\end{proof}

We add another lower bound that shows an instance in which the optimal point in the  optimal trajectory doesn't get close to $0$ but still there is a linear regret.

\begin{theorem}\label{thm:dimension lower bound}
Assume $\epsilon \ge \frac{4\eta^2}{c_1d^\barconst}$ and $\barconst$-barrier regularizer. There is a set of constant losses for which there is an $\epsilon$-approximate OMD trajectory that achieves a regret of $\Omega(T)$.
\end{theorem}
\begin{proof}
The losses are $\ell_t^d = 0$ and $\ell_t^i = 1$ for $i\in[d-1]$ for all $t$. We will show that having $w_t=w_1$ for all $t\in[T]$ is a valid $\epsilon$-approximate OMD trajectory. Since $w_1$ is the uniform distribution, the total loss is $T-\frac{T}{d}$. The optimal point is $w^*=(0,\dots,0,1)$, namely having $1$ only in the $d$th coordinate, which gives a total loss of $0$. Since $T-\frac{T}{d}=\Omega(T)$ even for $d=2$, this seals the proof.

We will now prove by induction that if $w_t=w_1$, $w_1$ is an $\epsilon$-approximate minimizer for $\phi_t$. Denote:
\begin{align*}
    \tilde{w}_{t+1} = \arg\min_{w\in\Delta_d}\phi_t(w)
\end{align*}
From \Cref{lem:max step} with $\epsilon = 0$ we get:
\begin{align*}
    \abs{w_t^d-\tilde{w}_{t+1}^d} &\le \frac{4\eta}{r''(w_t^d)} \le \frac{4\eta}{c_1d^\barconst} \le \frac{\epsilon}{\eta} \\
\end{align*}
Since $w_t^d=1/d$:
\begin{align*}
    \tilde{w}_{t+1}^d &\le \frac{1}{d} + \frac{\epsilon}{\eta}
\end{align*}
Summing for all coordinates:
\begin{align*}
    \frac{d-1}{d} - \frac{\epsilon}{\eta} &\le \sum_{i=1}^{d-1}\tilde{w}_{t+1}^i =  \inprod{\ell_t}{\tilde{w}_{t+1}}
\end{align*}

Since $\inprod{\ell_t}{w_1} = \frac{d-1}{d}$ we have:
We have:
\begin{align*}
    \inprod{\eta\ell_t}{\tilde{w}_{t+1}} &\ge \inprod{\eta\ell_t}{w_1}  - \epsilon\\
\end{align*}

Since by definition $D_R(w_t,w_t) \le D_R(\tilde{w}_{t+1},w_t)$, we get:
\begin{align*}
    \phi_t(w_t) \le \phi_t(\tilde{w}_{t+1}) + \epsilon
\end{align*}
which means that $w_t=w_1$ is an $\epsilon$-minimizer, as needed.
\end{proof}

\begin{theorem}
Consider the following instance with negative entropy regularizer for some $k\le\frac{T\eta}{20}$. For the first $\frac{3k}{2\eta}$ steps, the loss is $(0,1)$. Then, for the next $\frac{k}{\eta}$ steps, the loss is $(1,0)$. Then, for the rest ($\ge \frac{3T}{4}$) of the steps, the loss is $(0,1)$. There is an error OMD trajectory with $\epsilon=4\eta e^{-k/2}$ that has a regret $\Omega(T)$.
\end{theorem}
\begin{proof}
After $\tau=\frac{k}{2\eta}$ steps we have $w_\tau^2 \le \frac{\epsilon}{4\eta}$. From \Cref{lem:stuck}, it is a valid error trajectory if for every $3\tau \ge t\ge\tau$, $w_t =w_\tau$. 

On the steps between $3\tau$ and $4\tau$ we have a loss of $(1,0)$. Since $w_{3\tau}=w_\tau$, we have that $w_{4\tau} = \roundy{\frac{1}{2}, \frac{1}{2}}$. That is because this is what would have happen if those last $\tau$ steps where after $\tau$ (as the sum of losses for both coordinates is $\tau$), and since we didn't move at all in $\tau\le t\le 3\tau$ it is the same.

On the steps between $4\tau$ and $5\tau$ we assume no errors. Coordinate $1$ does the same trajectory that coordinate $2$ did in the beginning, so we have $w_\tau^1\le\frac{\epsilon}{4\eta}$. 

From \Cref{lem:stuck}, it is a valid error trajectory if for every $T \ge t\ge5\tau$, $w_t =w_{5\tau} \le \frac{\epsilon}{4\eta}$. Since this are $3T/4$ steps, we have a regret of $\Theta(T)$.

For summary:
\begin{align*}
    w_1 &= \roundy{\frac{1}{2}, \frac{1}{2}} \\
    w_\tau &\approx \roundy{1-\frac{\epsilon}{4\eta}, \frac{\epsilon}{4\eta}} \\
    w_{3\tau} &\approx \roundy{1-\frac{\epsilon}{4\eta}, \frac{\epsilon}{4\eta}} \\
    w_{4\tau} &= \roundy{\frac{1}{2}, \frac{1}{2}} \\
    w_{5\tau} &\approx \roundy{\frac{\epsilon}{4\eta}, 1-\frac{\epsilon}{4\eta}} \\
    w_T &\approx \roundy{\frac{\epsilon}{4\eta}, 1-\frac{\epsilon}{4\eta}}
\end{align*}
\end{proof}

\section{Proof of \Cref{lem:balance upper bound}}\label{sec:simplex upper bound}

\begin{lemma} \label{lem:bregman monotonic}
Let $w_1,w_2\in(0,1]$ such that $w_1 \le w_2$, then $D_r(0, w_1) \le D_r(0, w_2)$
\end{lemma}
\begin{proof}
Denote $f(x) = D_r(0, x)$. We have:
\begin{align*}
    f(x) &= r(0) - r(x) + r'(x)x \\
    f'(x) &= -r'(x) + r''(x)x + r'(x) = r''(x)x \ge 0
\end{align*}
Which means that $f$ is increasing in $(0, 1]$.
\end{proof}

\begin{lemma}\label{lem:not_too_far_phi}
Let $\hat{\gamma}=(\curly{\hat{w}_t}_{t=1}^T, \curly{\ell_t}_{t=1}^T, R, \eta)$ be an $\epsilon$-approximate OMD trajectory with $\eta \le \frac{1}{4}$ and coordinate separable $R$ with $r''(w) = 1/w^\barconst$. For every $i\in[d]$ and $t\in[T]$ such that $\epsilon \le \frac{\eta^2}{r''(\hat{w}_t)})$ we have:
\begin{align*}
    (\nabla\phi_t(\hat{w}_t^i)) \le O\roundy{2^\barconst\eta} 
\end{align*}
\end{lemma}
\begin{proof}
Since $\ell_t^i \le 1$ we have $\eta\ell_t^i \le \eta$, which means that we only need to prove:
\begin{align*}
    r'(\hat{w}_t^{i}) - r'(\hat{w}_{t-1}^{i}) \le O\roundy{2^\barconst\eta}
\end{align*}
Since $r'$ is monotonically increasing it is trivial if $\hat{w}_t^i \le \hat{w}_{t-1}^i$, continuing assuming $\hat{w}_t^i > \hat{w}_{t-1}^i$. We have $\epsilon \le \eta^2/r''(\hat{w}_t^i) \le 1/(16r''(\hat{w}_t^{i}))$, so from \Cref{lem:not_too_far}:
\begin{align*}
    \hat{w}_{t-1}^{i}&\ge \frac{1}{2}\hat{w}_{t}^{i}\\
    \Leftrightarrow \frac{2^\barconst}{(\hat{w}_t^{i})^\barconst} &\ge \frac{1}{(\hat{w}_{t-1}^{i})^\barconst}\\
    \Leftrightarrow 2^\barconst r''(\hat{w}_t^{i}) &\ge r''(\hat{w}_{t-1}^{i})
\end{align*}

From \Cref{lem:max step}:
\begin{align*}
    \hat{w}_{t}^{i} - \hat{w}_{t-1}^{i} \le \frac{4\eta}{r''(\hat{w}_{t}^{i})} + \sqrt{\frac{\epsilon}{r''(\hat{w}_{t}^{i})}}
\end{align*}

Which implies:
\begin{align*}
\epsilon &\le \frac{\eta^2}{\hat{w}_{t}^{i}} \le \frac{\eta^2}{r''(\hat{w}_t^{i})}\\
\Rightarrow \sqrt{\frac{\epsilon}{r''(\hat{w}_{t}^{i})}} &\le  \frac{\eta}{r''(\hat{w}_{t}^{i})}\\
    \Rightarrow \hat{w}_{t}^{i} - \hat{w}_{t-1}^{i} &\le \frac{5\eta}{r''(\hat{w}_{t}^{i})}
\end{align*}

From mean value theorem and monotonicity of $r''$: 
\begin{align*}
    r'(\hat{w}_{t}^{i}) - r'(\hat{w}_{t-1}^{i}) &\le \abs{\hat{w}_{t}^{i} - \hat{w}_{t-1}^{i}}\max_{w\in\curly{\hat{w}_t^{i},\hat{w}_{t-1}^{i}}} r''(w) 
    \\
    &\le \roundy{\hat{w}_{t}^{i} - \hat{w}_{t-1}^{i}}r''(\hat{w}_{t-1}^{i})
    \\
    &\le \roundy{\hat{w}_{t}^{i} - \hat{w}_{t-1}^{i}}2^\barconst r''(\hat{w}_{t}^{i^*})
    \\
    &\le \frac{5\eta}{r''(\hat{w}_{t}^{i})}2^{\barconst}r''(\hat{w}_{t}^{i}) \\
    &\le 5\cdot2^\barconst\eta\\
    &= O(2^\barconst \eta)
\end{align*}
\end{proof}

\begin{lemma} \label{lem:LB best UB regret}
Let $\cK=\Delta_d$ and $\hat{\gamma}=(\curly{\hat{w}_t}_{t=1}^T, \curly{\ell_t}_{t=1}^T, R, \eta)$ with $\eta \le \frac{1}{16}$, coordinate separable $R$ with $r''(w) = 1/w^\barconst$ and uniform initialization $\hat{w}_1=(1/d\dots1/d)$ be an $\epsilon$-approximate OMD trajectory such that there is $\xi>0$ such that for every $t\in[T]$, $\hat{w}_t^{i^*}\ge\xi$. If $\epsilon \le \frac{\eta^4}{r''\roundy{\min\curly{\frac{\eta}{d}, \xi}}}$, its regret w.r.t any $w\in\cK$ is bounded by:
\begin{align*}
    \reg(w) \le \frac{1}{\eta}D_R(w,\hat{w}_1) + O(2^\barconst T\eta)
\end{align*}
\end{lemma}
\begin{proof}
Let $\xi'=\min\curly{\frac{\eta}{d}, \xi}$, and let $S_t = \curly{i\ne i^*\;\colon\;\hat{w}_t^i \ge \xi'}$ for $t \ge 2$.

In every step $t$ we set $\tilde{w}_t$ to be:
\begin{align*}
    \tilde{w}_t^i &= \hat{w}_t^i &i\notin S_t,i\ne i^* \\
    \tilde{w}_t^i &= 0 &i\in S_t \\
    \tilde{w}_t^{i^*} &= 1-\sum_{i\notin S_t 
    }\hat{w}_t^i
\end{align*}

Since the changes between $\hat{w}_t$ and $\tilde{w}_t$ are only in coordinates with value greater then $\xi'$, the effective smoothness is upper bounded by $r''(\xi')$ (since $r''(w) = 1/w^\barconst$ for all $w\in(0,1]$). To use \Cref{lem:epsilon optimality conditions}, we need to show that $\epsilon \le D^2r''(\xi') /2$ where $D$ is the diameter w.r.t to $L_1$ norm. Indeed, we have that $r''(w^i) \ge 1$ for all $w\in \Delta_d$ and $i\in[d]$ and $D=2$. By our assumptions it holds that $\epsilon \le 1$, hence $\epsilon \le D^2r''(\xi')/2$. Thus, from \Cref{lem:epsilon optimality conditions} on $\phi_t$:
\begin{align*}
    \inprod{\eta\ell_{t-1} + \nabla R(\hat{w}_{t}) - \nabla R(\hat{w}_{t-1})}{\tilde{w}_t - \hat{w}_t} \ge -2\sqrt{2r''(\xi')\epsilon} \ge -2\eta^2.
\end{align*}

Which means:
\begin{align}
&\roundy{\eta\ell_{t-1}^{i^*} +  \nabla R(\hat{w}_{t})^{i^*} - \nabla R(\hat{w}_{t-1})^{i^*}}(\tilde{w}_t^{i^*} - \hat{w}_t^{i^*}) \notag\\
&\quad + \sum_{i\in S_t}\roundy{\eta\ell_{t-1}^{i} +  \nabla R(\hat{w}_{t})^{i} - \nabla R(\hat{w}_{t-1})^{i}}(0 - \hat{w}_t^{i}) \notag\\
&\ge -2\eta^2.\label{eq:optimality on S}
\end{align}

Notice that since $\xi' \le \frac{\eta}{d}$ we have that $\sum_{i\notin S_t 
    }\hat{w}_t^i \le \eta$ 
    which means $1 - \tilde{w}_t^{i^*} \le \eta$. Additionally, from \Cref{lem:not_too_far_phi} we have that $\nabla\phi(\hat{w}_t^{i^*}) \le O(2^\barconst\eta)$. We have:
\begin{align*}
    \roundy{\eta\ell_{t-1}^{i^*} + \nabla R(\hat{w}_{t})^{i^*} - \nabla R(\hat{w}_{t-1})^{i^*}}\roundy{\tilde{w}_t^{i^*} - 1} &= -O(2^\barconst\eta^2)
\end{align*}

Thus, \Cref{eq:optimality on S} can be written as:
\begin{align*}
&\roundy{\eta\ell_{t-1}^{i^*} +  \nabla R(\hat{w}_{t})^{i^*} - \nabla R(\hat{w}_{t-1})^{i^*}}(1 - \hat{w}_t^{i^*}) + \\&\quad\roundy{\eta\ell_{t-1}^{i^*} +  \nabla R(\hat{w}_{t})^{i^*} - \nabla R(\hat{w}_{t-1})^{i^*}}(\tilde{w}_t^{i^*} - 1) + \\&\quad\sum_{i\in S}\roundy{\eta\ell_{t-1}^{i} +  \nabla R(\hat{w}_{t})^{i} - \nabla R(\hat{w}_{t-1})^{i}}(0 - \hat{w}_t^{i}) \\
    &\ge -\eta^2 \\
    &\Rightarrow\\
    &\roundy{\eta\ell_{t-1}^{i^*} +  \nabla R(\hat{w}_{t})^{i^*} - \nabla R(\hat{w}_{t-1})^{i^*}}(1 - \hat{w}_t^{i^*}) + \sum_{i\in S}\roundy{\eta\ell_{t-1}^{i} +  \nabla R(\hat{w}_{t})^{i} - \nabla R(\hat{w}_{t-1})^{i}}(0 - \hat{w}_t^{i}) \ge -O(2^\barconst\eta^2)\\
    &\eta\ell_{t-1}^{i^*}(\hat{w}_t^{i^*}-1) + \eta\sum_{i\in S}\ell_{t-1}^i(\hat{w}_t^i - 0) \le \\
    &\quad\roundy{\nabla R(\hat{w}_{t})^{i^*} - \nabla R(\hat{w}_{t-1})^{i^*}}(1 - \hat{w}_t^{i^*}) + \sum_{i\in S}\roundy{  \nabla R(\hat{w}_{t})^{i} - \nabla R(\hat{w}_{t-1})^{i}}(0 - \hat{w}_t^{i}) + O(2^\barconst\eta^2)
\end{align*}

From \Cref{lem:three_points}:
\begin{align*}
    \eta\ell_{t-1}^{i^*}(\hat{w}_t^{i^*}-1) + \eta\sum_{i\in S_t}\ell_{t-1}^i(\hat{w}_t^i -0) 
    &\le 
    D_r(1, \hat{w}_{t-1}^{i^*}) - D_r(1,\hat{w}_t^{i^*}) - D_r(\hat{w}_{t}^{i^*},  \hat{w}_{t-1}^{i^*}) \\ &+\sum_{i\in S_t}D_r(0, \hat{w}_{t-1}^i) - D_r(0,\hat{w}_t^i) - D_r(\hat{w}_{t}^i,  \hat{w}_{t-1}^i)\\
    &+ O(2^\barconst\eta^2)
\end{align*}

Fix some coordinate $i\ne i^*$, and let $(s_1,t_1),(s_2,t_2),\dots (s_n,t_n)$ be all enter and exit times for $i$ to $S_t$. Namely, for every $j\in[n]$ and $s_j\le t \le t_j$, $i\in S_t$, and $i\notin S_t$ otherwise. 
Hence, 
\begin{align*}
    \sum_{t:i\in S_t}D_r(0, \hat{w}_{t-1}^i) - D_r(0,\hat{w}_t^i) 
    &= 
    \sum_{j=1}^n D_r(0, \hat{w}_{s_j-1}^i) - D_r(0, \hat{w}_{t_j}^i),
\end{align*}
where the equality follows by telescoping the terms.
Since $s_j$ is enter time for coordinate $i$, we have that $i\notin S_{s_j-1}$, which means that $\hat{w}_{s_j-1}^i< \xi'$. On the other hand, $i\in S_{t_j}$, which means that $\hat{w}_{t_j}^i \ge \xi' > \hat{w}_{s_j-1}^i$. 
Thus, by \Cref{lem:bregman monotonic} we get that $D_r(0,\hat{w}_{s_j-1}) \le D_r(0, \hat{w}_{t_{j}})$  which we apply on the RHS of the previous display to obtain:
\begin{align*}
    \sum_{t:i\in S_t}D_r(0, \hat{w}_{t-1}^i) - D_r(0,\hat{w}_t^i) &\le D_r(0, \hat{w}_{s_1-1}^i)
\end{align*}

We now argue that for every $i\in[d]$, $i\in S_2$, which means that $s_1=2$. Assume by contradiction that $\hat{w}_2^i < \hat{w}_1^i$ (and thus $r''(\hat{w}_2^i) > r''(\hat{w}_1^i)$, from \Cref{lem:max step}:
\begin{align*}
    \hat{w}_2^i -\hat{w}_1^i &\le \frac{\eta}{r''(1/d)} + \sqrt{\frac{\epsilon}{r''(1/d)}}\\
    &\le \frac{\eta}{d} + \sqrt{\frac{1}{16r''(\eta/d)r''(1/d)}}\\
    &\le \frac{1}{4d} + \frac{1}{4d}\\
    \Rightarrow \hat{w}_2^i &\ge 1/2d
\end{align*}

Thus:
\begin{align*}
        \sum_{t:i\in S_t}D_r(0, \hat{w}_{t-1}^i) - D_r(0,\hat{w}_t^i) &\le D_r(0, \hat{w}_{1}^i)
\end{align*}

Thus:
\begin{align}
    \sum_{t=2}^T & \eta\ell_{t-1}^{i^*}(\hat{w}_t^{i^*}-1) + \eta\sum_{i\in [d]\setminus i^*}\sum_{t:i\in S_t}\ell_{t-1}^i(\hat{w}_t^i-0) 
    \notag\\
    &\le D_r(1, \hat{w}_1^{i^*}) - \sum_{t=2}^TD_r(\hat{w}_t^{i^*},\hat{w}_{t-1}^{i^*}) +\sum_{i\in[d]\setminus i^{*}}D_r(0, \hat{w}_1^i) - \sum_{t=2}^TD_r(\hat{w}_t^i,\hat{w}_{t-1}^i)
    + O(2^\barconst T\eta^2)\notag\\
    &= D_R(w^*,\hat{w}_1) - \sum_{t=2}^TD_R(\hat{w}_t,\hat{w}_{t-1})  + O(2^\barconst T\eta^2)\label{eq:in S}
\end{align}
Additionally, since if $i\notin S_t$ we have $\hat{w}_t^i \le \xi' \le \frac{\eta}{d}$, we can say:
\begin{align}\label{eq:non S}
    \sum_{i\in [d]\setminus i^*}\sum_{t:i\notin S_t}\ell_{t-1}^i(\hat{w}_t^i-0) \le T\eta
\end{align}

Combining \Cref{eq:non S,eq:in S} (recall that $w^*$ has $1$ in $i^*$ and $0$ in other coordinates):
\begin{align*}
\sum_{t=2}^T\eta\ell_{t-1}^{i^*}(\hat{w}_t^{i^*}-1) + \eta\sum_{i\in [d]\setminus i^*}\sum_{t=2}^T\ell_{t-1}^i(\hat{w}_t^i-0) &\le D_R(w^*,\hat{w}_1) - \sum_{t=2}^TD_R(\hat{w}_t,\hat{w}_{t-1})  + O(2^\barconst T\eta^2)\label{eq:in S}\\
    \Longleftrightarrow\sum_{t=2}^T\inprod{\eta\ell_{t-1}}{\hat{w}_t-w^*} &\le D_R(w^*,\hat{w}_1) - \sum_{t=2}^TD_R(\hat{w}_t,\hat{w}_{t-1})  + O(2^\barconst T\eta^2)\\
    \Longleftrightarrow\sum_{t=2}^T\inprod{\ell_{t-1}}{\hat{w}_t-w^*} &\le \frac{1}{\eta}D_R(w^*,\hat{w}_1) - \frac{1}{\eta}\sum_{t=2}^TD_R(\hat{w}_t,\hat{w}_{t-1})  + O(2^\barconst T\eta)\\
\end{align*}

From \Cref{lem:omd helper}:
\begin{align*}
    \reg(w^*) \le \frac{1}{\eta}D_R(w^*,\hat{w}_1) + O(2^\barconst T\eta)
\end{align*}
\end{proof}

\begin{lemma}\label{lem:balance to bounded optimal arm}
Let $\gamma=(\curly{w_t}_{t=1}^T, \curly{\ell_t}_{t=1}^T, R, \eta)$ and $\hat{\gamma}=(\curly{\hat{w}_t}_{t=1}^T, \curly{\ell_t}_{t=1}^T, R, \eta)$ be OMD trajectory and OMD error trajectory, and assume $T\ge 4$, $\gamma$ is $k$-balanced and $\epsilon\le\frac{1}{r''\roundy{\frac{1}{2de^{k+1}}}T^2}$. 

Then, for every $t\in[T]$, $\hat{w}_t^{i^*} \ge \frac{1}{de^{k+1}}$.
\end{lemma}
\begin{proof}
We will prove by induction on $t$. Since $\hat{w}_1^i = \frac{1}{d}$ the base case holds.

Assume the statement is true for $t-1$ and we prove for $t$. 

If $\hat{w}_t^{i^*} \ge \hat{w}_s^{i^*}$ for some $s < t$ then the claim follows from the inductive assumption. If for every $i\ne i^*$, $\hat{w}_t^i \le \frac{1}{d}$ we have that $\hat{w}_t^{i^*} \ge \frac{1}{d}$ and the claim follows trivially. Proceeding, we consider the case that for every $s < t$, $\hat{w}_t^{i^*}< \hat{w}_s^{i^*}$ and there is some $i\in[d]$ such that $\hat{w}_t^i > \frac{1}{d}$.

Since $T\ge 3$ we have from the Lemma's assumptions and the induction assumptions that $\epsilon \le \frac{1}{18r''(1/de^{k+1})} \le \frac{1}{16r''(\hat{w}_{t-1}^i)}$. From \Cref{lem:not_too_far}:
% %
\begin{align*}
    \hat{w}_{t}^{i^*}\ge \frac{\hat{w}_{t-1}^{i^*}}{2}
\end{align*}
From this and the inductive assumption we have that for all $s\in[1,t]$, $\hat{w}_s^{i^*} \ge \frac{1}{2de^{k+1}}$. (We now want to improve this statement to $\hat{w}_t^{i^*} \ge \frac{1}{de^{k+1}}$.)

Fix $i$ to be the coordinate for which $\hat{w}_t^i > \frac{1}{d}$. We'll show that for every $s\in[1,t]$, $\hat{w}_s^i > \frac{1}{de^{k+1}}$. Assume by contradiction that $s$ is the last time $\hat{w}_s^i \le \frac{1}{de^{k+1}}$. Again, since $T\ge 3$ we have $\epsilon \le 1/16r''(\hat{w}_{s+1}^i)$, thus from \Cref{lem:not_too_far}:
\begin{align*}
    \hat{w}_s^i \ge \frac{1}{2}w_{s+1}^i > \frac{1}{2de^{k+1}}
\end{align*}

Which means that for every $s'\in\squary{s,t}$, $\hat{w}_{s'}^i \ge \frac{1}{2de^{k+1}}$. From \Cref{lem:simplex_trajectory_diff} and our assumption on $\epsilon$ (see \Cref{def:balance_simplex} for the definition of $B^i$):
\begin{align*}
    B_{\hat{\gamma}}^i(s,t) \le B_{{\gamma}}^i(s,t) + T\sqrt{r''\roundy{\frac{1}{2de^{k+1}}}\epsilon} \le k+1
\end{align*}
(we used $r''$ because in our case, that $c_1=c_2=1$, it is the same).

Recall that $\hat{w}_t^{i^*} < \hat{w}_s^{i^*}$, from \Cref{lem:balance to frac}:
\begin{align*}
    \hat{w}_t^i \le e^{B_{\hat{\gamma}}^i(s,t)}\hat{w}_{s}^i \le e^{k+1}\hat{w}_{s}^i \le e^{k+1}\frac{1}{de^{k+1}} = \frac{1}{d}
\end{align*}
Which is a contradiction to $\hat{w}_t^i > \frac{1}{d}$. Now we can continue assuming that for all $s\in [1,t]$, $\hat{w}_s^i > \frac{1}{de^{k+1}}$.

From \Cref{lem:trajectory diff}:
\begin{align*}
    B_{\hat{\gamma}}^i(1,t) \le B_{{\gamma}}^i(1,t) + T\sqrt{r''\roundy{\frac{1}{2de^{k+1}}}\epsilon} \le k+1
\end{align*}

Now, from \Cref{lem:balance to frac} (recall that $\hat{w}_t^i > \frac{1}{d} = \hat{w}_1^i$):
\begin{align*}
    \hat{w}_t^{i^*} \ge \frac{\hat{w}_1^{i^*}}{e^{k+1}} = \frac{1}{de^{k+1}},
\end{align*}
which completes the inductive step and the proof.
\end{proof}

\begin{lemma*}[Restatement of \Cref{lem:balance upper bound}]
Let $\cK=\Delta_d$ be the simplex, let $\curly{\ell_t}_{t=1}^T$ be an $\alpha$-balanced loss sequence, and let $R(w)=\sum_{i=1}^d w_i\log w_i$ be the negative entropy regularizer. Assume $\eta \le 1/16$ and $T\ge 3$,
if the approximation error satisfies 
\[
    \epsilon \le \frac{1}{d \max\curly{6e^{\eta \alpha}, 1/\eta}}\min\curly{\eta^4,1/T^2},
\]
then the regret of any $\epsilon$-approximate OMD trajectory is bounded as
\[
    \mathrm{Regret}(w) \le \frac{1}{\eta} D_R(w, w_1) + O(T \eta).
\]
\end{lemma*}
\begin{proof}
From \Cref{lem:polytope balance} the optimal trajectory is $\alpha\eta$-balanced. Since $\epsilon \le \frac{1}{16de^{\alpha\eta+1}T^2} = \frac{1}{r''(1/16de^{\alpha\eta+1})T^2}$, from \Cref{lem:balance to bounded optimal arm} for every $t\in [T]$, $\hat{w}_t^{i^*} \ge \frac{1}{de^{\alpha\eta +1}}$. 

We also have that $\epsilon \le \frac{\eta^4}{r''(\min\curly{\eta/d, 1/de^{\alpha\eta+1})}}$, hence 
from \Cref{lem:LB best UB regret} with $\xi=1/de^{\alpha\eta+1}$ and $\barconst=1$ we get the desired results.
\end{proof}

\opt{
\paragraph{Implications for the main theorems.}
\begin{itemize}
    \item \textit{\Cref{thm:up_barrier_adversarial}.}  
    Any adversarial sequence over the simplex is $T$-balanced: if one coordinate exceeds the best by more than $T$, it must actually be the best. 
    Applying \Cref{lem:balance upper bound} with $\alpha=T$ gives the desired upper bound. 

    \item \textit{\Cref{thm:ub_barrier_stochastic}.}  
    For i.i.d.\ losses, Hoeffding's inequality and union bounds implies that with probability at least $1-\delta$, the balance is at most $O(\sqrt{T\log(dT^2/\delta)})$. 
    Plugging this value of $\alpha$ into \Cref{lem:balance upper bound} yields the stochastic upper bound.
\end{itemize}

}

\section{Proof of \Cref{thm:ub_nonentropy}}\label{sec:polytope upper bound}
\begin{lemma}\label{lem:gradient_diff}
For every $\barconst$-barrier regularizer $r$and $1 \ge w_2 \ge w_1 \ge 0$ we have:
\begin{align*}
    r'(w_2) - r'(w_1) 
    &= \frac{c_1}{w_1^{\barconst-1}} - \frac{c_1}{w_2^{\barconst-1}} \\
\end{align*}
\end{lemma}
\begin{proof}
\begin{align*}
    r'(w_2) - r'(w_1) &= \int_{w_1}^{w_2} r''(w)dw \\
    &\ge \int_{w_1}^{w_2}\frac{c_1}{w^\barconst}dw \\
    &= \frac{c_1}{w_1^{\barconst-1}} - \frac{c_1}{w_2^{\barconst-1}} 
\end{align*}
\end{proof}

\begin{lemma}\label{lem:nonentropy_iterate_lb}
Denote $\psi = \roundy{\frac{c_1}{8\eta Td + c_1(2d)^{\barconst-1}}}^{1/(\barconst-1)}$. In the assumptions of \Cref{thm:ub_nonentropy}, for every $t,i$:
\begin{align*}
    w_t^i \ge \psi
\end{align*}
\end{lemma}
\begin{proof}
One can see that the assumptions of the theorem are that $\epsilon \le \eta^4\min\curly{c_2,\frac{1}{c_2}}\roundy{\frac{\psi}{2}}^{\barconst}$. We will now prove by induction that for every $t\in[T]$, $w_t^i \ge \psi$.

Notice that $\eta \le \frac{1}{16c_1}$ and $\epsilon \le \eta\psi^{\barconst} \le \frac{(w_{t-1}^i)^\barconst}{16c_1}$, which means that from \Cref{lem:not_too_far}, we know for start that $w_t^i \ge \psi/2$. This means that for every $t'\in[1,t]$, $w_t^i \ge \psi/2$. One can see that $\epsilon \le c_2\psi/4$ which means that we can use \Cref{lem:trajectory diff} with $\psi/2$. Since the balance of an exact trajectory is always bounded by $\eta T$, for every normalized $v\in \ker(A)$:
\begin{align*}
    B_\gamma^v(1,t)
    &\le T\eta + T\sqrt{c_2\epsilon\frac{2}{\psi}^{\barconst}} \\
    &\le 2\eta T
\end{align*}

From \Cref{lem:balance to gradient UB} and the induction assumption, for every $i\in [d]$ and $t'\le t$: 
\begin{align*}
    r'(1/2d)-r'(w_t^i) \le 8\eta Td
\end{align*}
If $w_t^i \le 1/2d$ we can use \Cref{lem:gradient_diff}:
\begin{align*}
\frac{c_1}{\roundy{w_i^{t'}}^{\barconst-1}} &\le 8\eta Td + c_1(2d)^{\barconst-1}\\
\Rightarrow w_i^{t'} &\ge \roundy{\frac{c_1}{8\eta Td + c_1(2d)^{\barconst-1}}}^{1/(\barconst-1)} = \psi
\end{align*}
Else, i.e if $w_t^i \ge 1/2d$, we have:
\begin{align*}
    w_t^i &\ge 1/2d\\
    &= \roundy{\frac{1}{(2d)^{\barconst-1}}}^{1/\barconst-1}\\
    &= \roundy{\frac{c_1}{c_1(2d)^{\barconst-1}}}^{1/\barconst-1}\\
    &\ge \roundy{\frac{c_1}{8\eta Td + c_1(2d)^{\barconst-1}}}^{1/\barconst-1}\\
    &= \psi
\end{align*}
Which ends the induction step.
\end{proof}

\textbf{Proof of \Cref{thm:ub_nonentropy}}: 
Since the polytope is a subset of the simplex, the diameter is bounded by $2$. From \Cref{lem:nonentropy_iterate_lb}, the effective smoothness of the trajectory is bounded by $\beta \coloneqq c_2/\psi^\barconst$. By the assumption about $\epsilon$ we have $\epsilon \le \eta^4/\beta$. From \Cref{lem:epsilon optimality conditions}, for every $t$:
\begin{align*}
    \inprod{\eta\ell_t + \nabla R(w_{t+1}) - \nabla R(w_t)}{w^* - w_{t+1}} \ge -2\sqrt{2\beta\epsilon} = \Theta\roundy{\eta^2}
\end{align*}

From here it is straightforward standard OMD arguments:
\begin{align*}
\eta \ell_t \cdot (w_{t+1}-w^*) 
&\leq (\nabla R(w_{t+1}) - \nabla R(w_{t})) \cdot (w^*-w_{t+1}) + \Theta\roundy{\eta^2}
\\
&=
D_R(w^*,w_t) - D_R(w^*,w_{t+1}) - D_R(w_{t+1},w_t) + \Theta\roundy{\eta^2} 
\end{align*}
Summing for all $t\in[T]$:
\begin{align*}
    \eta\sum_{t=1}^T \ell_t \cdot (w_{t+1}-w^*)
&\leq D_R(w^*,w_1) - \sum_{t=1}^T D_R(w_{t+1},w_t) + \Theta\roundy{\eta^2 T}\\
\implies \sum_{t=1}^T \ell_t \cdot (w_{t+1}-w^*) &\le \frac1\eta D_R(w^*,w_1) - \frac1\eta \sum_{t=1}^T D_R(w_{t+1},w_t) + \Theta\roundy{\eta T}
\end{align*}
From \Cref{lem:omd helper}:
\begin{align*}
    \reg(w^*) &\le O\roundy{\frac{1}{\eta}D_R(w^*, w_1) + \Theta\roundy{\eta T}}
\end{align*} 

\hfill\rule{2mm}{2mm}

\section{Proof of \Cref{thm:lb_nonsimplex}}\label{sec:polytope lower bound}
\subsection{Polytope definition}
The polytope is defined as $\curly{w \in \rr^d: Aw=b\land \roundy{w_i\ge 0,\,\forall i\in[d]}}$ for $A,b$ defined below. Denote $m=16\log\roundy{\frac{1}{\epsilon}}$, we have $d=5m+2$. Additionally, we for assume for convenience that $m \ge 128\log(2T)$ and $\epsilon < 4\eta$ (obviously proof that works for small $\epsilon$ works for bigger).

The matrix $A$ has $4m+1$ rows. The first $4m$ rows are, for every $i\in[m]$:
\begin{align*}
    A_i &= e_{m+i} + e_{2m+i} - 2e_{3m+i} \\
    A_{m+i} &= e_{m+i} - e_{2m+i} \\
    A_{2m+i} &= e_i + 3e_{m+i} + e_{4m+i}\\
    A_{3m+i} &= e_{4m+i} - e_{5m+1}
\end{align*}
The last row is:
\begin{align*}
    A_{4m+1} = e_{5m+1} + e_{5m+2}
\end{align*}

And:
\begin{align*}
    b = A\sum_{i=1}^d\frac{1}{d}e_i
\end{align*}
Namely, $b$ is defined such that the point $\roundy{\frac{1}{d},\frac{1}{d},\dots,\frac{1}{d}}$ is in the polytope. Denote this point as $w_1$, the OMD will always start from here.

Denote the following set of $m+1$ vectors, $\curly{v_i}_i^{m+1}$:
\begin{align*}
    v_i &= 3e_i - e_{m+i} - e_{2m+i} - e_{3m+i} \quad\forall i\in[m]\\
    v_{m+1} &= \sum_{i=1}^m e_i + e_{5m+2} - \sum_{i=4m+1}^{5m+1}e_i
\end{align*}

\begin{lemma}\label{lem:A ker}
$\curly{v_i}_{i=1}^{m+1}$ is a basis for $\ker(A)$.
\end{lemma}
\begin{proof}
One can notice that $A$ is already in echelon form, so it is full ranked, which means that $dim \,\ker(A) = m+1$. Additionally, Every vector of $v$ has a non-zero coordinate that's zeroed in all other vectors of $v$, so $v$ is linear independent, which means that we only need to show that each of the vectors indeed nulls $A$.

For every $i\in[1,m]$, $v_i$ has common non-zero coordinates only with $A_i,A_{m+i},A_{2m+i},A_{3m+i}$. One can easily see that it nulls them. As for $v_{m+1}$, it has common non-zero coordinates with $A_i$ for every $i\in[2m+1,4m+1]$, which again can be seen easily to nullify.
\end{proof}

\begin{lemma}
The polytope is a subspace of the simplex
\end{lemma}
\begin{proof}
To be inside the simplex all points of the polytope should satisfy two conditions - all coordinates greater than $0$ and the sum of coordinates should be $1$. The first is by definition in this polytope. 

Let $w$ be some point in the polytope. Since $Aw=b$ and $Aw_1=b$, $w-w_1\in \ker(A)$. From \Cref{lem:A ker}, we can write:
\begin{align*}
    w = w_1 + \sum_i\alpha_iv_i
\end{align*}
For some $\alpha_i\in\rr$.

All the vectors in $v$ has the sum of their coordinates $0$. Thus, the sum of coordinates of $w$ is the same as $w_1$, concluding the proof.
\end{proof}

\subsection{General settings and hardness event}
Since we want to prove a lower bound of the form $T\sqrt{\frac{\eta}{\log\roundy{\frac{1}{\epsilon}}}} = \Theta\roundy{T\sqrt{\frac{\eta}{d}}}$, and there's a known lower bound for $T\eta$, we can assume $d\le \frac{1}{\eta}$.

The losses for the first $m$ coordinates is constant $0$, for the $[m+1,4m]$ coordinates it's constant $1$, for the $[4m+1,5m]$ coordinates it's gaussian with mean $0$ and variance $1$, for the $5m+1$th coordinate it is guassian with mean $\sqrt{\eta d}\le 1$ and variance $1$ and for the $5m+2$th coordinate it's constant $0$.

Denote $\tau=\frac{3}{\eta}$. We define the hardness event $E$ to be the following events:
\begin{align*}
    \sum_{i=4m+1}^{5m+1}\ell_{:\tau}^i &\le 0\\
    \sum_{i=4m+1}^{5m+1}\ell_t^i &\le \frac{m}{16}\quad\forall t\in[T]
\end{align*}
\begin{lemma}
\begin{align*}
    \Pr\roundy{E} = \Omega(1)
\end{align*}
\end{lemma}
\begin{proof}
Denote $G=\sum_{i=4m+1}^{5m+1}\ell_{:\tau}^i$. Since $G$ is a sum of gaussian random variables, it is also a gaussian random variable, denote its mean with $\mu$ and variance $\sigma^2$. Simple calculation shows that $\mu=\sqrt{d\eta}\tau=2\sqrt{\frac{d}{\eta}}$ and $\sigma^2=\tau(m+1)=\frac{2(m+1)}{\eta}$. Since $m=\Theta(d)$, we have that $\mu=\Theta(\sigma)$. It is a general attribute of a gaussian that in such case the probability of having $G\le 0$ is $\Theta(1)$.

Fix $t\in[T]$. Using Hoeffding inequality we have that w.p $\frac{1}{2T}$:
\begin{align*}
    \sum_{i=4m+1}^{5m+1}\ell_t(i) \le \sqrt{\frac{m+1}{2}\log\roundy{2T}}
\end{align*}
Since $\log(2T)\le \frac{m}{128}$:
\begin{align*}
    \sum_{i=4m+1}^{5m+1}\ell_t(i) \le \frac{m}{16}
\end{align*}
Union bound on all $t\in[T]$ concludes the proof.
\end{proof}

For every $t\in[T]$ we know that $w_t-w_1\in \ker(A)$. From \Cref{lem:A ker}, it can be written as a linear combination of $v$. Denote the coefficients as $\alpha$, namely:
\begin{align*}
    w_t = w_1 + \sum_{i=1}^{m+1}\alpha_t^iv_i
\end{align*}

\subsection{Analysis}
\begin{lemma} \label{lem:first cors high}
For every $i\in[m]$, $w_\tau^i \ge \frac{5}{2d}$
\end{lemma}
\begin{proof}
Assume by contradiction that $w_\tau^i \ge \frac{3}{d}$. One can notice that the $(m+i,2m+i,3m+i)$ are only in $v_i$ with the same coefficient, which means that $w_t^{m+i}=w_t^{2m+i}=w_t^{m+i}$. 

Notice that $\eta\inprod{\ell_{:\tau}}{v_i} = -3\eta\tau = -9$. From \Cref{lem:polytope balance}:
\begin{align*}
    -9 &= \inprod{\nabla R(w_1) - \nabla R(w_t)}{v_i} \\
    &= 3\log\roundy{\frac{w_1^i}{w_\tau^i}} + \log\roundy{\frac{w_\tau^{m+i}}{w_1^{m+i}}} + \log\roundy{\frac{w_\tau^{2m+i}}{w_1^{2m+i}}} + \log\roundy{\frac{w_\tau^{3m+i}}{w_1^{3m+i}}} \\
    &\ge 3\log\roundy{\frac{1}{3}} + 3\log\roundy{\frac{w_\tau^{m+i}}{w_1^{m+i}}}\\
    \Longleftrightarrow -3 &\ge \log\roundy{\frac{1}{3}} + \log\roundy{dw_\tau^{m+i}}\\
    &= \log\roundy{\frac{dw_\tau^{m+i}}{3}} \\
    \Longleftrightarrow w_\tau^{m+i} &\le \frac{3}{e^{3}d} \le \frac{1}{6d}
\end{align*}

Which means that $\alpha_i \ge \frac{5}{6d}$. Additionally, $\alpha_{m+1} \ge -\frac{1}{d}$, since else it violates $e_{5m+2} \ge 0$. We have:
\begin{align*}
    w_\tau^i = \frac{1}{d} + 3\alpha_i + \alpha_{m+i} \ge \frac{5}{2d}
\end{align*}
\end{proof} 

\begin{lemma} \label{lem:polytope stuck base}
Assume $E$ and optimal trajectory, we have $\alpha_\tau^{m+1} \le -\frac{1}{d} + e^{m/8}$
\end{lemma}
\begin{proof}
We'll first show that $\alpha_\tau^{m+1}\le 0$. Assume the opposite by contradiction. This means that $w_\tau^i \le w_1^i$ for all $i\in[4m+1,5m+1]$ and $w_\tau^{5m+2} \ge w_1^{5m+2}$. Together with \Cref{lem:first cors high} this means that in the positive elements of $v_{m+1}$ we have $w_\tau^i> w_1^i$ and in the negative elements we have $w_\tau^i< w_1^i$. This means that $\inprod{\nabla R(w_1) - \nabla R(w_t)}{v_{m+1}} < 0$. From $E$ we have $\inprod{\ell_{:\tau}}{v_{m+1}} \ge 0$, which is a contradiction to \Cref{lem:polytope balance}.

We continue assuming $\alpha_\tau^{m+1}\le 0$. Notice that $\alpha_\tau^{m+1} \ge -\frac{1}{d}$ to satisfy the constraint $w_\tau^{5m+2} \ge 0$, so for every $i\in[4m+1,5m+1]$, we have $w_\tau^i \le \frac{1}{2d}$. From $E$ and \Cref{lem:polytope balance}:
\begin{align*}
    0 &\le -\sum_{i=4m+1}^{5m+1}\ell_{:\tau}^i\\
    &= \sum_{i=1}^{m}\ell_{:\tau}^i + \ell_{:\tau}^{5m+2} - \sum_{i=4m+1}^{5m+1}\ell_{:\tau}^i\\
    &= \inprod{\ell_{:\tau}}{v_{m+1}}
\end{align*}
From \Cref{lem:polytope balance}:
\begin{align*}
    0 &\le \inprod{\nabla R(w_1) - \nabla R(w_t)}{v_{m+1}} \\
    &= \sum_{i=1}^m\log\roundy{\frac{w_1^i}{w_\tau^i}} + \log\roundy{\frac{w_1^{5m+2}}{w_\tau^{5m+2}}} - \sum_{i=4m+1}^{5m+1}\log\roundy{\frac{w_1^i}{w_\tau^i}}
\end{align*}
Since $w_\tau^i \le \frac{1}{2d}$, we have $\frac{w_1^i}{w_\tau^i} \ge 2$. From \Cref{lem:first cors high}, we have $\frac{w_1^i}{w_\tau^i} \le 2.5$. Thus:
\begin{align*}
    m\log(2.5) - (m+1)\log(2) + \log\roundy{\frac{w_1^{5m+2}}{w_\tau^{5m+2}}} &\ge 0
\end{align*}

Since $m \ge 8$, $(m+1)\log(2) \le m\log(2.2)$. Thus:
\begin{align*}
    m\log(2.5/2.2) + \log\roundy{\frac{w_1^{5m+2}}{w_\tau^{5m+2}}} &\ge 0\\
    \Longleftrightarrow w_\tau^{5m+2} &\le e^{m/8}
\end{align*}

Since $w_\tau^{5m+2}$ can only be altered with $\alpha_\tau^{m+1}$, this concludes the proof.
\end{proof}

\begin{lemma} \label{lem:polytope stuck step}
For some $t\in[T]$, assume $\alpha_{t-1}^{m+1} \le -\frac{1}{d} + e^{-m/8}$ and $E$. There is an $\epsilon$-approximate step for which $\alpha_{t}^{m+1} \le -\frac{1}{d} + e^{-m/8}$.
\end{lemma}
\begin{proof}
First we show that for the optimal step, $\alpha_{t}^{m+1} \le -\frac{1}{d} + e^{-m/16}$

If $\alpha_t^{m+1} \le \alpha_{t-1}^{m+1}$ the proof concludes from the assumption. Continuing assuming the opposite. This means that for every $i\in[4m+1,5m+1]$, $w_t^i \le w_{t-1}^i$.

Additionally, we'll show that for every $i\in[1,m]$, $w_t^i\ge w_{t-1}^i$. Assume otherwise for some $i$, since $\alpha_t^{m+1} \ge \alpha_{t-1}^{m+1}$ it means that $\alpha_t^i \le \alpha_{t-1}^i$, which means that for every $j\in\curly{m+i,2m+i,3m+i}$ we have $w_t^{j}\ge w_{t-1}^{j}$. This means that $\inprod{\nabla R(w_{t-1})-\nabla R(w_t)}{v_{i}} \ge 0$, which contradicts \Cref{lem:polytope balance} (as $\inprod{\ell_t}{v_i}$ has a constant value of $-1$).

From the second part of $E$ and \Cref{lem:polytope balance}:
\begin{align*}
    -\frac{m}{16} &\le \eta\inprod{\ell_t}{v_{m+1}} \\
    &= \inprod{\nabla R(w_{t-1})-\nabla R(w_t)}{v_{m+1}}\\
    &= \sum_{i=1}^m\log\roundy{\frac{w_{t-1}^i}{w_t^i}} + \log\roundy{\frac{w_{t-1}^{5m+2}}{w_t^{5m+2}}} - \sum_{i=4m+1}^{5m+1}\log\roundy{\frac{w_{t-1}^i}{w_t^i}}\\
     &\le \log\roundy{\frac{w_{t-1}^{5m+2}}{w_t^{5m+2}}}\\
     \Longleftrightarrow w_t^{5m+2} &\le w_{t-1}^{5m+2}e^{m/16} \\
     &\le e^{-m/16}
\end{align*}

Since the $5m+2$th coordinate is controlled only by $v_{m+1}$, this concludes the fact that $\alpha_{t}^{m+1} \le -\frac{1}{d} + e^{-m/16} = -\frac{1}{d} + \epsilon$, which means that $\alpha_{t}^{m+1} \le \alpha_{t-1}^{m+1} + \epsilon$.

Next we argue that in the optimal step, for every $i\in[1,m]$, $\alpha_t^i \ge \alpha_{t-1}^i$. Assume otherwise for some $i$. This means that for every $j\in\curly{m+i,2m+i,3m+i}$ we have $w_t^{j}\ge w_{t-1}^{j}$. Additionally, it means that $w_t^i \le w_{t-1}^i + \epsilon$.

From \Cref{lem:polytope balance}:
\begin{align*}
    -3\eta &= \eta\inprod{\ell_t}{v_{i}}\\
    &= \inprod{\nabla R(w_{t-1})-\nabla R(w_t)}{v_{i}}\\
    &= 3\log\roundy{\frac{w_{t-1}^i}{w_t^i}} + \log\roundy{\frac{w_t^{m+i}}{w_{t-1}^{m+i}}} + \log\roundy{\frac{w_t^{2m+i}}{w_{t-1}^{2m+i}}} + \log\roundy{\frac{w_t^{3m+i}}{w_{t-1}^{3m+i}}}\\
    &\ge 3\log\roundy{\frac{w_{t-1}^i}{w_t^i}}\\
    &\ge 3\log\roundy{\frac{w_t^i-\epsilon}{w_t^i}}\\
    \Longrightarrow -\eta &\ge \log\roundy{1 - \frac{\epsilon}{4}} \ge -\frac{\epsilon}{4} \\
    \epsilon &\ge 4\eta
\end{align*}
Which contradicts our assumption that $\epsilon < 4\eta$.

By now we showed that if the $t$th step is optimal, we have $\alpha_{t-1}^{m+1} \le \alpha_{t}^{m+1} \le \alpha_{t-1}^{m+1} + \epsilon$ and for all $i\in[m]$, $\alpha_t^i \ge \alpha_{t-1}^i$.

We next argue that if we keep the same $\alpha_{t}^i$ for all $i\in[m]$ but change $\alpha_{t}^{m+1}$ to be equal to $\alpha_{t-1}^{m+1}$, this will be an $\epsilon$-approximate step.

First we notice that all $w_t$ is now closer to $w_{t-1}$, which means that the bregman divergence only shrinks from that change. Indeed, coordinates $[4m+1,5m+2]$ are only getting closer from the change, coordinates $[m+1,4m]$ doesn't change (the change in $v_{m+1}$ doesn't affect them). Finally, since for all $i\in[m]$, $\alpha_t^i \ge \alpha_{t-1}^i$, we still have $w_t^i \ge w_{t-1}^i$, which means that those coordinates also got closer.

The proof concludes from the fact that from the second part of $E$, the first term in the objective can only be changed in $\frac{m\eta\epsilon}{16} < \epsilon$.

\end{proof}

\begin{theorem}
There is an $\epsilon$-approximate trajectory that get a regret of:
\begin{align*}
    \Omega\roundy{T\sqrt{\frac{\eta}{\log\roundy{\frac{1}{\epsilon}}}}}
\end{align*}
\end{theorem}
\begin{proof}
From \Cref{lem:polytope stuck base,lem:polytope stuck step} we get that there is an $\epsilon$-approximate trajectory such that for every $t \ge \tau$, $\alpha_t^{m+1} \le -\frac{1}{d} + e^{-m/8} \le 0$, which means that $w_t^{5m+1} \ge \frac{1}{d}$. The total expected loss of this coordinate is $T\sqrt{d\eta}$, which means that this trajectory suffers a loss of $\Omega\roundy{T\sqrt{\frac{\eta}{\log\roundy{\frac{1}{\epsilon}}}}}$.

Now we only need to show that there is a point in the polytope that gets zero loss. Indeed, one can see that if $\alpha^i = \frac{1}{d}$ for all $i\in[1,m+1]$, the point $w=w_1 + \sum_i\alpha^iv_i$ has all coordinates with non-zero loss ($[m+1,4m], 5m+1$) to be zeroed.
\end{proof}

\section{Approximate FTRL}\label{sec:ftrl}

In this section we analyze an approximate version of the Follow The Regularized Leader (FTRL) algorithm, analogous to the approximate OMD algorithm analyzed in the paper, shown in \cref{alg:ftrl}.
Our simple analysis shows that $\epsilon$-approximate updates in FTRL gives rise to an additional additive $O(\sqrt{\epsilon}T)$ term in the regret, implying that polynomially small error suffices for optimal regret performance (i.e., $\epsilon = O(1/T)$ for $O(\sqrt{T})$ regret).

\begin{algorithm}[ht] 
    \caption{Approximate Follow The Regularized Leader} \label{alg:ftrl} 
	\begin{algorithmic}[1]
		\State Input: $\eta > 0$, regularization function ${R}$, and a bounded, convex and closed set $\cK$.
		\State Let $\tilde{w}_1  = \arg\min_{w \in \cK} {\left\{ {R}(w)\right\} }$.
		\For{$t=1$ to $T$}
		\State Play $\tilde{w}_t$ and observe loss $\ell_t$
        \State Denote
        \begin{align*}
            \phi_t(w) = \eta\sum_{s=1}^t \inprod{w_t}{\ell_t} + {R}(w)
        \end{align*}
		\State Update
		\begin{align*}
			\phi_t(\tilde{w}_{t+1}) \le \min_w \phi_t(w) + \epsilon
		\end{align*}
		\EndFor
	\end{algorithmic}
\end{algorithm}

Our analysis closely follows the standard FTRL analysis, and specifically, arguments appearing in \citet{hazan2016introduction}. 
For convenience, we denote $f_t(w) = \inprod{\ell_t}{w}$ for $t\ge 1$ and $f_0 = (1/\eta) R$. Additionally, we denote $w_{t+1}$ to be the exact minimizer of $\phi_{t}$ over $\mathcal{K}$. We will need the following ``be the leader'' lemma, the proof of which can be found in \citet{hazan2016introduction}.

\begin{lemma}[Lemma 5.4 in \citealp{hazan2016introduction}]\label{lem:ftl}
For every $w^*\in\cK$:
\begin{align*}
    \sum_{t=0}^T f_t(w^*) \ge \sum_{t=0}^T f_t(w_{t+1})
    .
\end{align*}
\end{lemma}
The key fact about the approximate minimizers $\tilde w_t$ is the following.
\begin{lemma}\label{lem:ftrl_movement}
For every $t\in[T]$:
\begin{align*}
    \norm{\tilde{w}_t - w_{t+1}} \le 2\eta + \sqrt{2\epsilon}
    .
\end{align*}
\end{lemma}
\begin{proof}
Since $w_{t+1}$ is the minimizer of $\phi_t$, from first order optimality conditions:
\begin{align*}
    \inprod{\nabla\phi_t(w_{t+1})}{\tilde{w}_t - w_{t+1}} \ge 0
    .
\end{align*}
Thus:
\begin{align*}
    D_{\phi_t}(\tilde{w}_t, w_{t+1}) &= \phi_t(\tilde{w}_t) - \phi_t(w_{t+1}) - \inprod{\nabla \phi_t(w_{t+1})}{\tilde{w}_t - w_{t+1}} \\
    &\le \phi_t(\tilde{w}_t) - \phi_t(w_{t+1})
    .
\end{align*}
Since a linear term doesn't change the Bregman divergence, this also upper bounds $D_R(\tilde{w}_t, w_{t+1})$. Thus:
\begin{align*}
    D_R(\tilde{w}_t, w_{t+1}) &\le \phi_t(\tilde{w}_t) - \phi_t(w_{t+1})\\
    &= \phi_{t-1}(\tilde{w}_t) - \phi_{t-1}(w_{t+1}) + \eta \inprod{\ell_t}{\tilde{w_t} - w_{t+1}}\\
    &\le \epsilon + \eta \inprod{\ell_t}{\tilde{w_t} - w_{t+1}}\\
    &\le \epsilon + \eta \norm{\ell_t}_*\norm{\tilde{w_t} - w_{t+1}}\\
    &\le \epsilon + \eta \norm{\tilde{w_t} - w_{t+1}}
    .
\end{align*}
From the $1$-strong convexity of $R$:
\begin{align*}
    D_R(\tilde{w}_t, w_{t+1}) &\ge \frac{1}{2}\norm{\tilde{w}_t - w_{t+1}}^2\\
    \implies 2\epsilon + 2\eta \norm{\tilde{w_t} - w_{t+1}} &\ge\norm{\tilde{w}_t - w_{t+1}}^2 \\
    \implies \norm{\tilde{w}_t - w_{t+1}} &\le 2\eta + \sqrt{2\epsilon}\tag{\Cref{lem:solve quadratic}}
    .
\end{align*}
\end{proof}

We can now prove the main result of this section:

\begin{theorem}
The regret of \Cref{alg:ftrl} is bounded by:
\begin{align*}
    \reg(w^*) \le \frac{R(w^*) - R(w_1)}{\eta} + (2\eta + \sqrt{2\epsilon}) T
    .
\end{align*}
\end{theorem}

\begin{proof}
\begin{align*}
    \reg(w^*) &= \sum_{t=1}^T\inprod{\ell_t}{\tilde{w}_t - w^*}\\
    &= \sum_{t=1}^T\inprod{\ell_t}{\tilde{w}_t - w^*} + \frac{R(w^*) - R(w^*)}{\eta} \\
    &\le \sum_{t=1}^T\inprod{\ell_t}{\tilde{w}_t - w_{t+1}} + \frac{R(w^*) - R(w_1)}{\eta}\tag{\Cref{lem:ftl}}\\
    &\le \sum_{t=1}^T \norm{\ell_t}_*\norm{\tilde{w}_t - w_{t+1}} + \frac{R(w^*) - R(w_1)}{\eta}\\
    &\le \roundy{2\eta + \sqrt{2\epsilon}}T + \frac{R(w^*) - R(w_1)}{\eta}\tag{\Cref{lem:ftrl_movement}}
    .
\end{align*}
\end{proof}

\end{document}